\documentclass[12pt,letterpaper,fleqn,oneside]{article}

\usepackage{sectsty}
\usepackage{caption}
\usepackage{times}
\usepackage{amssymb,amsfonts,amsmath,amscd}
\usepackage[pdftex,colorlinks]{hyperref}
\usepackage[pdftex]{graphicx}
\usepackage{bm}
\usepackage{mathtools}
\usepackage{array}
\usepackage{booktabs}
\usepackage{verbatim}
\usepackage{fancyhdr}
\usepackage{pdfpages}
\usepackage{bm}
\usepackage{amsthm}
\newtheorem{theorem}{Theorem}
\newtheorem{lemma}{Lemma}

\usepackage[]{hyperref}
\hypersetup{colorlinks=false,}

\makeatletter
\renewcommand*\env@matrix[1][\arraystretch]{%
  \edef\arraystretch{#1}%
  \hskip -\arraycolsep
  \let\@ifnextchar\new@ifnextchar
  \array{*\c@MaxMatrixCols c}}
\makeatother

\usepackage{tikz}
\usepackage{tikz-3dplot}

\usepackage{pgfplots}
\usepackage{filecontents}
\usepackage[font=footnotesize,labelfont=bf]{caption}

\newcommand{\UTIASprogram}{N/A}
\newcommand{\UTIAStitle}{On Observability and Identifiability of Tightly-coupled Ultrawideband-aided Inertial Localization}
\newcommand{\UTIASdocument}{TR}
\newcommand{\UTIASrevision}{Rev: 1.0}
\newcommand{\UTIASauthor}{Author}

\hypersetup{%
    pdftitle={\UTIASdocument: \UTIAStitle},
    pdfauthor={\UTIASauthor},
    pdfkeywords={},
    pdfsubject={\UTIASprogram},
    pdfstartview={},
    urlcolor=cyan,
    linkcolor=red,
}%

\title{\sf\bfseries \UTIAStitle }

\author{Abhishek Goudar \\
       Institute for Aerospace Studies \\
       University of Toronto\\
       \texttt{<abhishek.goudar@robotics.utias.utoronto.ca>} \\
       Supervisor: Dr. ~A.~P.~Schoellig
       }
\date{}

\begin{document}

\fancypagestyle{plain}{%
    \fancyhf{}%
    \fancyfoot[C]{}%
    \fancyhead[R]{\begin{tabular}[b]{r}\small\sf \UTIASdocument\\
        \small\sf\UTIASrevision\\
        \small\sf\today \end{tabular}}%
    \fancyhead[L]{\includegraphics[height=0.45in]%
        {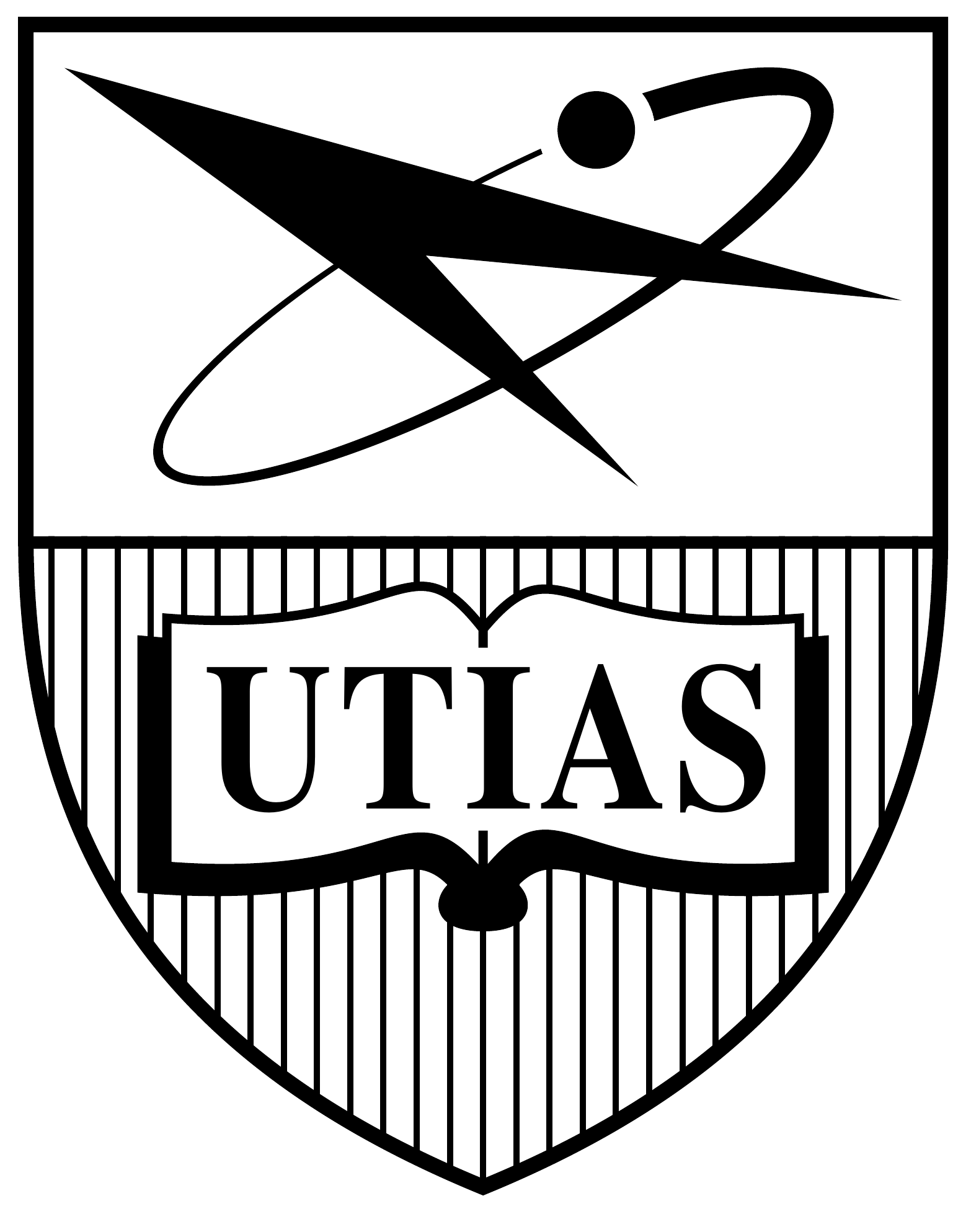} \begin{tabular}[b]{l}\sc{University of Toronto} \\ \sc{Institute for Aerospace Studies} \\ \mbox{} \end{tabular}}%
    \renewcommand{\headrulewidth}{0pt}
    \renewcommand{\footrulewidth}{0pt}}

\pagestyle{fancy}

\allsectionsfont{\sf\bfseries}

\renewcommand{\captionlabelfont}{\sf\bfseries}


\lhead{ \includegraphics[height=0.45in]%
        {figs/utias.pdf} \begin{tabular}[b]{l}\sc{University of Toronto} \\ \sc{Institute for Aerospace Studies} \\ \mbox{} \end{tabular}}
\rhead{ \begin{tabular}[b]{r}\small\sf \UTIASdocument\\
        \small\sf\UTIASrevision\\
        \small\sf\today \end{tabular}}
\chead{}
\lfoot{}
\cfoot{\thepage}
\rfoot{}

\renewcommand{\headrulewidth}{0pt}
\renewcommand{\footrulewidth}{0pt}


\maketitle%

\begin{abstract}
The combination of ultrawideband (UWB) radios and inertial measurement units (IMU) can provide accurate positioning. To ensure reliable communication, the radios are generally mounted at the extremities of a mobile system whereas the IMUs are located closer to the center of gravity for use in control, resulting in a spatial offset between the IMU and the UWB radio. Additionally, data from heterogeneous sensors can arrive at different time instants. The systematic fusion of data from multiple sources requires the temporal offset and spatial offset between the sensors to be known. 

An important aspect of calibration is the observability of the system state and identifiability of the system parameters. Estimating the state or parameters of a system that is otherwise unobservable or unidentifiable, can result in poor estimates. In this report, the local weak observability of the state and the identifiability of the temporal offset for a tightly-coupled UWB-aided inertial localization system is studied.
\end{abstract}

\newpage
\tableofcontents

\newpage
\section{Introduction}
A typical UWB-based positioning system consists of UWB radios, known as \textit{anchors}, installed in the surrounding environment. A mobile system equipped with a UWB radio calculates its position by measuring the time of flight between its UWB radio and the anchors. Generally, IMUs and UWB radios are not co-located and there is a spatial offset between the two sensors, also referred to as \textit{sensor extrinsic} parameters. This spatial offset can affect positioning accuracy.

Additionally, data from heterogeneous sensors are generally timestamped with different sources of clocks which results in a \emph{temporal} offset. Most sensor fusion schemes require data from different sensors to have timestamps with respect to a single clock. This is generally done through hardware synchronization by using a common clock signal for multiple sensors. The next best choice is software synchronization, with a server running on a destination computer and a client on the sensor hardware. However, many off-the-shelf components do not support either of these methods. Estimating the state of a mobile system without compensating for the sensor extrinsic parameters or the temporal offset can result in poor positioning accuracy, particularly for large offsets.

An important aspect of calibration is the ability to unambiguously recover the relevant system states given system outputs. This can be accomplished using \emph{observability} and \emph{identifiability} analysis. In this report, we analyze the observability of the core state, including the sensor extrinsic parameters and the identifiability of the temporal offset parameter.

\begin{figure}[h]
	\centering
	\tdplotsetmaincoords{70}{110}
	\begin{tikzpicture}[scale=1.3,tdplot_main_coords] 
		\coordinate (WORLD_ORIGIN) at (0,0,0);
		\coordinate (UWB_ORIGIN) at (3,4.05,3.5);
		\coordinate (IMU_ORIGIN) at (3,6.0,4);
		\coordinate (ANCHOR_ORIGIN) at (0,1,4);
		\coordinate (RADIO_ORIGIN) at (0,1,3.85);
		\coordinate (DRONE_ORIGIN) at (3,6,3.6);

		\node [anchor=north] at (WORLD_ORIGIN){$\{W\}$};
		\node [anchor=north] at (IMU_ORIGIN){$\{I\}$};
		\node [anchor=north east] at (UWB_ORIGIN){$\{U\}$};
		\node [anchor=north east] at (ANCHOR_ORIGIN){$\{A_j\}$};

		\draw[solid] (UWB_ORIGIN) -- (IMU_ORIGIN) node[above, pos=0.45]{$\textbf{p}^I_U$};
		\draw[dashed] (WORLD_ORIGIN) -- (IMU_ORIGIN) node[below right, pos=0.4]{$\{\textbf{p}^W_I, \textbf{q}^W_I\}$};
		\draw[dashed] (WORLD_ORIGIN) -- (ANCHOR_ORIGIN) node[above, pos=0.5]{$\textbf{p}^W_j$};
		\draw[dotted] (ANCHOR_ORIGIN) -- (UWB_ORIGIN) node[above right, pos=0.45]{$r_j$};

		\draw[thick,->] (WORLD_ORIGIN) -- (0,1,0) node[anchor=north west]{$x$};
		\draw[thick,->] (WORLD_ORIGIN) -- (-1.5,0,0) node[anchor=south west]{$y$};
		\draw[thick,->] (WORLD_ORIGIN) -- (0,0,1) node[anchor=south]{$z$};

		\tdplotsetrotatedcoords{0}{0}{0}
		\tdplotsetrotatedcoordsorigin{(UWB_ORIGIN)}\draw[thick,color=black,tdplot_rotated_coords,->] (0,0,0) --(-0.7,0,0) node[anchor=south]{$x$};
		\draw[thick,color=black,tdplot_rotated_coords,->] (0,0,0) --(0,-0.5,0) node[anchor=south]{$y$};
		\draw[thick,color=black,tdplot_rotated_coords,->] (0,0,0) --(0,0,0.5) node[anchor=south]{$z$};

		\tdplotsetrotatedcoords{0}{0}{0}
		\tdplotsetrotatedcoordsorigin{(IMU_ORIGIN)}\draw[thick,color=black,tdplot_rotated_coords,->] (0,0,0) --(-.7,0,0) node[anchor=south]{$x$};
		\draw[thick,color=black,tdplot_rotated_coords,->] (0,0,0) --(0,-0.5,0) node[anchor=south]{$y$};
		\draw[thick,color=black,tdplot_rotated_coords,->] (0,0,0) --(0,0,0.5) node[anchor=south]{$z$};

		\tdplotsetrotatedcoords{0}{0}{0}
		\tdplotsetrotatedcoordsorigin{(ANCHOR_ORIGIN)}\draw[thick,color=black,tdplot_rotated_coords,->] (0,0,0) --(-.7,0,0) node[anchor=south]{$x$};
		\draw[thick,color=black,tdplot_rotated_coords,->] (0,0,0) --(0,-0.5,0) node[anchor=south]{$y$};
		\draw[thick,color=black,tdplot_rotated_coords,->] (0,0,0) --(0,0,0.5) node[anchor=south]{$z$};
		
		\node[inner sep=0pt, opacity=0.1] (drone) at (DRONE_ORIGIN) {\includegraphics[width=.35\textwidth]{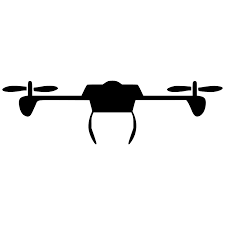}};
		\node[inner sep=0pt, opacity=0.7] (mobradio) at (UWB_ORIGIN) {\includegraphics[width=.01\textwidth]{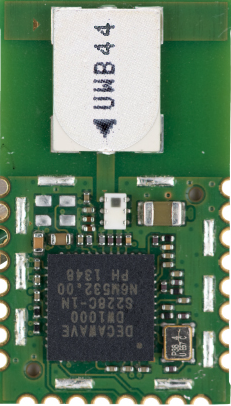}};
		\node[inner sep=0pt, opacity=0.3] (anchor) at (RADIO_ORIGIN) {\includegraphics[width=.02\textwidth]{figs/dw1000.png}};
	\end{tikzpicture}
	\caption{The relationship between different frames involved in calibration of sensor extrinsics. Frame $\{W\}$ corresponds to a gravity-aligned world reference frame. IMU and mobile radio reference frames are represented by $\{I\}$ and $\{U\}$, respectively. The offset of mobile radio in the IMU frame is $\textbf{p}_U^I$. The frame affixed to the phase center of $j^{th}$ anchor is $\{A\}_j$ and its position in world reference frame is $\textbf{p}^W_j$. The pose of the IMU in world frame is {$\{\textbf{p}^W_I, \textbf{q}^W_I\}$}.}
	\label{fig:frame_setup}
\end{figure}

\section{System Modelling}
We define the following coordinate frames for the setup in Fig. \ref{fig:frame_setup}:
\begin{enumerate}
\item \textbf{world frame} $\{W\}$, a local absolute reference frame, in which the pose of the IMU and the positions of individual anchors are expressed.
\item \textbf{mobile radio frame} $\{U\}$, a frame affixed to the phase center of the mobile radio antenna. The phase center is the point on the antenna from which the electromagnetic radiation spreads spherically outward.
\item \textbf{IMU frame} $\{ I \}$, a frame corresponding to the IMU body center, in which the body accelerations and angular velocities are measured.
\item \textbf{anchor frame} $\{A\}_j$, a frame affixed to the phase center of the $j^{th}$ anchor.
\end{enumerate}

The relationship between the different frames of reference is shown in Fig. \ref{fig:frame_setup}. Next, we describe the system parameterization, sensor models, motion model, and observation model.

\subsection{System parameterization}
The system is parameterized by the following 20 dimensional state vector:
\begin{equation}
\mathbf{x}(t) = (\mathbf{p}^W_I(t), \mathbf{v}^W_I(t), \mathbf{q}^W_I(t), \mathbf{b}_a(t), \mathbf{b}_\omega(t), \mathbf{p}^I_U(t), t_d), \label{eqn:state} 
\end{equation}
\noindent where, $\{ \mathbf{p}^W_I(t), \mathbf{v} ^W_I(t), \mathbf{q}^W_I(t) \}$ denote the position, translational velocity and orientation of the IMU frame with respect to the world frame. A unit quaternion parameterization is used for representing orientations. We use quaternions for their algebraic properties and singularity-free orientation representation. Accelerometer and gyroscope biases are denoted by $\mathbf{b}_a(t)$ and $\mathbf{b}_g(t)$ and constitute the intrinsic parameters of the IMU. The sensor extrinsic parameter $\mathbf{p}^I_U(t)$ is the position of the mobile radio expressed in the IMU frame and $t_d$ is the temporal offset between UWB and IMU sampling instants.

\subsection{Gyroscope and accelerometer model} \label{sec:imu_model}
The measured angular rate by a single-axis gyroscope $\omega_m$ is related to the true angular rate $\omega_t$ as:
\begin{equation}
\omega_m = \omega_t + b_\omega + n_\omega,
\label{eqn:gyro_model}
\end{equation}
where $b_\omega$ is a time-varying bias and $n_\omega$ is assumed to be a realization of an additive white Gaussian noise (AWGN) source with covariance $N_\omega$ i.e. $n_\omega \sim \mathcal{N}(0, N_\omega)$.  The bias is modelled as being driven by a AWGN noise source: $\dot{b}_\omega = n_{b\omega}$, where $n_{b\omega} \sim \mathcal{N}(0, N_{b\omega})$. A similar model is used for the individual axes of a triaxial gyroscope: $\bm{\omega}_m = (\omega_x, \omega_y, \omega_z)$, $\mathbf{b}_\omega = (b_{\omega x}, b_{\omega y}, b_{\omega z})$, $\mathbf{n}_{b\omega} = (n_{b \omega x}, n_{b \omega y}, n_{b \omega z})$, $\mathbf{Q}_\omega = \text{diag}([\sigma_{\omega x}^2, \sigma_{\omega y}^2, \sigma_{\omega z}^2])$ and $\mathbf{Q}_{b\omega} = \text{diag}([\sigma_{b\omega x}^2, \sigma_{b\omega y}^2, \sigma_{b\omega z}^2])$. The measured linear acceleration $a_m$ is related to the true linear acceleration $a_t$ as:
\begin{equation}
a_m = a_t + b_a + n_a,
\label{eqn:accel_model}
\end{equation}
where $b_a$ is a time-varying bias and $n_a$ is assumed to be an AWGN source of covariance $N_a$ i.e. $n_a \sim \mathcal{N}(0, N_a)$. As before, the bias is modelled as being driven by another AWGN noise source $n_{ba} \sim \mathcal{N}(0, N_{ba})$: $\dot{b}_a = n_{ba}$. The model is extended to a triaxial accelerometer with: $\mathbf{a}_m = (a_x, a_y, a_z)$, $\mathbf{b}_a = (b_{ax}, b_{ay}, b_{az})$, $\mathbf{n}_{ba} = (n_{b a x}, n_{b a y}, n_{b a z})$, $\mathbf{Q}_a = \text{diag}([\sigma_{a x}^2, \sigma_{a y}^2, \sigma_{a z}^2])$ and $\mathbf{Q}_{ba} = \text{diag}([\sigma_{bax}^2, \sigma_{bay}^2, \sigma_{baz}^2])$.

\subsection{Motion model}
The motion model in this work is a 3D kinematic motion model where the accelerometer and gyroscope measurements are used as control inputs:
\begin{align}
	&\dot{\mathbf{p}}^W_I = \mathbf{v}^W_I,&
	&\dot{\mathbf{q}}^W_I = \dfrac{1}{2}\bm{\Omega}(\bm{\omega}_t)\mathbf{q}^W_I,\label{eqn:non_linear_model_1}\\
	&\dot{\mathbf{v}}^W_I = \mathbf{R}^W_I \mathbf{a}_t - \mathbf{g}^W,&
	&\dot{\mathbf{b}}_a = \mathbf{n}_{ba},\label{eqn:non_linear_model_2}\\
	&\dot{\mathbf{b}}_g = \mathbf{n}_{b\omega},& 
	&\dot{\mathbf{p}}^I_U = \mathbf{0}_{3}, \label{eqn:non_linear_model_3} &\\
	&\dot{t_d} = 0, \label{eqn:non_linear_model_4}
\end{align}
where,
\begin{align*}
\bm{\Omega}(\bm{\omega}) &= \begin{bmatrix}
0 & -\bm{\omega}^T\\
\bm{\omega} & -[\bm{\omega}]_\times
\end{bmatrix},
\end{align*}
$\mathbf{g}^W = [0, 0, 9.8]^T m/s^2$ represents the acceleration due to gravity in the world frame, $\mathbf{R}^W_I := \mathbf{R}\{\mathbf{q}^W_I\}$ is the direction cosine matrix corresponding to the nominal orientation, and $[\cdot]_\times$ denotes the skew-symmetric cross-product matrix. 

\subsection{Observation model}
In a tightly-coupled system, the observation model is the distance between an anchor and the mobile radio. The measured distance to the $i^{th}$ anchor at time $t_r$ is:
\begin{equation}
    h(\mathbf{p}^W_i, \mathbf{x}(t_r)) = \| \mathbf{p}^W_i - \mathbf{p}^W_U(t_r) \|_2 + n_r(t_r),
    \label{eqn:meas_model}
\end{equation}
\noindent where $\mathbf{p}^W_i$ is the position of the $i^{th}$ static anchor, $\mathbf{p}^W_U(t_r) = \mathbf{R}^W_I(t_r) \mathbf{p}^I_U  + \mathbf{p}^W_I (t_r)$ is the position of the mobile radio at time $t_r$ and $\|.\|_2$ denotes the $\ell_2$ norm. The measurement noise $n_r(t_r)$ is assumed to be a zero-mean AWGN process, with covariance $\text{Q}_r$, i.e. $n_r(t_r) \sim \mathcal{N}(0, \text{Q}_r)$. Calculating \eqref{eqn:meas_model} requires the value of the state \eqref{eqn:state} at time $t_r$. In this paper, we use UWB time as reference time. Let $t_r$ be the timestamp of the latest UWB and IMU measurement. Note that due to the temporal offset $t_d$, the IMU measurement is actually generated at time $t_I = t_r - t_d$. To calculate the state at time $t_r$, we propogate the state at time $t_I$ for $t_d$ seconds using the motion model  \eqref{eqn:non_linear_model_1}-\eqref{eqn:non_linear_model_3} and the IMU measurement at time $t_I$. 

\section{Observability and Identifiability analysis}
The goal of observability analysis is to check if the state of a system can be determined uniquely given the outputs of a system. We decompose the problem of observability of the state~(\ref{eqn:state}) into \emph{(i)} \emph{local identifiability} of $t_d$ \cite{anguelova2008} and \emph{(ii)} \emph{local weak observability} \cite{Hermann1977} of the part of the state excluding $t_d$. The motivation and justification is provided below.

From (\ref{eqn:meas_model}), we see that $t_d$ can be modelled as a time-delay parameter. In \cite{anguelova2008}, it is shown that the identifiability of a single unknown constant time-delay can be analyzed independent of the observability of the state. The authors show that the identifiability of a single unknown constant time-delay in a nonlinear system is not directly related to the observability of other system states or parameters.  Specifically, the time-delay parameter can be identified directly from the \emph{input-output} representation \cite{anguelova2008} of the system. Thus, determining the identifiability of time-delay is equivalent to determining the existence of an input-output representation, which solely depends on the systems inputs, outputs, and their time-derivatives and not on the state. A necessary and sufficient condition for the existence of an input-output representation is the occurrence of the delayed input variables (in this case $\mathbf{a}_m$ and $\bm{\omega}_m$) in the output (\ref{eqn:meas_model}) (Theorem 2 in \cite{anguelova2008}). We use this approach to analyze the local identifiability of $t_d$.

As noted in \cite{anguelova2008}, identifiability of time-delay does not imply observability of the state. After proving the local identifiability of  $t_d$, we analyze the observability of the part of the state excluding $t_d$, $\widetilde{\mathbf{x}} = (\mathbf{p}^W_I, \mathbf{v}^W_I, \mathbf{q}^W_I, \mathbf{b}_a, \mathbf{b}_\omega, \mathbf{p}^I_U)$. For this, we use the method outlined in \cite{Hermann1977}, which involves determining the rank of the \emph{observability matrix} $\mathcal{O}$ \cite{Hermann1977}. The methods outlined in \cite{anguelova2008} and \cite{Hermann1977} consider the case of noise-free nonlinear systems. Hence, for analysis, we neglect the effect of noise parameters in the following sections. 

\subsection{Local identifiability of the temporal offset}
Following \cite{anguelova2008}, the local identifiablility of $t_d$ depends on whether it is present in the input-output representation of the system (\ref{eqn:non_linear_model_1})-(\ref{eqn:non_linear_model_3}) and (\ref{eqn:meas_model}), that is, the presence of delayed input variable(s) in the output function. In our case, it is sufficient to show that $\mathbf{a}_m(t_r - t_d)$ or $\bm{\omega}_m(t_r - t_d)$ appear in the measurement model. Without loss of generality, the measurement model (\ref{eqn:meas_model}) can be written as:
\begin{equation}
h(\mathbf{p}^W_i, \mathbf{x}(t_r)) = \frac{1}{2}\| \mathbf{p}^W_i - \mathbf{p}^W_U(t_r) \|^2_2,
\label{eqn:td_meas_model}
\end{equation}
where $\mathbf{p}^W_U(t_r)$ is the position of the mobile radio at time $t_r$ and $\| \cdot \|$ denotes the $\ell_2$ norm. We consider a single anchor as the analysis is identical for other anchors. The factor 1/2 is introduced for simplifying the analysis. Since there are two inputs, $t_d$ can be locally identified if either $\mathbf{a}_m$ or $\bm{\omega}_m$ is excited.  In the following theorem, we state the conditions under which $t_d$ is locally identifiable.

\begin{theorem}
The temporal offset parameter $t_d \in (0, T)$, for some $T \in \mathbb{R}$, is locally identifiable from the observation model (\ref{eqn:td_meas_model}) if:
\begin{description}
\item[(T1)] the mobile radio is not co-located with the anchor; and either {\normalfont{(T2)}} or {\normalfont{(T3)}} is satisfied:
\item[(T2)] at least one of $a_x$, $a_y$, or $a_z$ is excited; or
\item[(T3)] the mobile radio is not co-located with the IMU and all three of $\omega_x$, $\omega_y$, or $\omega_z$ are excited.
\end{description}
\label{lem:td_cond_1}
\end{theorem}




\begin{proof}
Expand $\mathbf{p}^W_U$ in measurement model \eqref{eqn:td_meas_model}:
\begin{align*}
h(\mathbf{p}^W_i, \mathbf{x}(t_r)) = \frac{1}{2}\| \mathbf{p}^W_i - \mathbf{R}^W_I (t_r) \mathbf{p}^I_U - \mathbf{p}^W_I(t_r) \|^2_2.
\end{align*}

To calculate $\mathbf{p}^W_I(t_r)$, we propogate the state forward in time by using the motion model \eqref{eqn:non_linear_model_1}-\eqref{eqn:non_linear_model_3}. We perform Euler integration by assuming a constant $\mathbf{a}_m$ for the duration of the integration $t_d$:
\begin{align*}
\mathbf{p}^W_I(t_r) \approx \mathbf{p}^W_I(t_I) + \mathbf{v}^W_I(t_I) t_d + \frac{1}{2}\mathbf{R}^W_I(t_I)(\mathbf{a}_m(t_I) - \mathbf{b}_a(t_I)) t_d^2.
\end{align*}

Rewriting $\mathbf{a}_m(t_I)$ as $\mathbf{a}_m(t_r - t_d)$ in the above expression yields:
\begin{equation}
\mathbf{p}^W_I(t_r) \approx \mathbf{p}^W_I(t_I) + \mathbf{v}^W_I(t_I) t_d + \frac{1}{2}\mathbf{R}^W_I(t_I)(\mathbf{a}_m(t_r - t_d) - \mathbf{b}_a(t_I)) t_d^2.
\end{equation}

Now we expand $\mathbf{R}^W_I (t_r) \mathbf{p}^I_U = \mathbf{R}^W_I (t_I + t_d) \mathbf{p}^I_U$. For a small angular-increment $\delta \mathbf{\theta} = \bm{\omega}_t t_d$, we have the following relation:
\begin{align*}
\mathbf{R}^W_I(t_I + t_d) &=  \mathbf{R}^W_I(t_I) e^{[\bm{\omega_t}(t_I) t_d]_\times},
\end{align*}
where $e^A$ denotes the matrix exponential for a matrix $A$ \cite{sola2017}. Approximating the matrix exponential to first-order:
\begin{align*}
e^{([\bm{\omega_m}(t_I) t_d]_\times)} \approx \mathbf{I}  + [\bm{\omega_t}(t_I) t_d]_\times + \mathcal{O}(t_d^2),
\end{align*}
where $\mathbf{I}$ is the identity matrix. Using $\bm{\omega}_t = \bm{\omega}_m - \mathbf{b}_\omega$ we have:
\begin{align*}
\mathbf{R}^W_I (t_r) \mathbf{p}^I_U &\approx 
\mathbf{R}^W_I(t_I) \left( \mathbf{I}  + [\bm{\omega_t}(t_I) t_d]_\times \right) \mathbf{p}^I_U \\
&= \mathbf{R}^W_I(t_I) \mathbf{p}^I_U + \left( \mathbf{R}^W_I(t_I) [\bm{\omega}_m(t_I) - \mathbf{b}_\omega(t_I)]_\times \right) \mathbf{p}^I_U ~ t_d,\\
&= \mathbf{R}^W_I(t_I) \mathbf{p}^I_U + \mathbf{R}^W_I(t_I) [\bm{\omega}_m(t_I) - \mathbf{b}_\omega(t_I)]_\times \mathbf{p}^I_U ~ t_d.
\end{align*}
Rewriting $\bm{\omega}_m(t_I)$ as $\bm{\omega}_m(t_r - t_d)$ in the above expression and using the relation yields:
\begin{equation}
\mathbf{R}^W_I(t_r) \mathbf{p}^I_U \approx \mathbf{R}^W_I(t_I) \mathbf{p}^I_U + \mathbf{R}^W_I(t_I) \mathbf{p}^I_U [(\bm{\omega}_m(t_r - t_d) - \mathbf{b}_\omega(t_I))]_\times ~ t_d.
\end{equation}

\noindent The required condition for $t_d$ to be locally identifiable is:
\begin{equation}
\dfrac{\partial h^{(k)}(\mathbf{p}^W_i, \mathbf{x}(t_r))}{\partial u^{(k)}(t_r - t_d)} \neq 0
\label{eqn:ide_cond}
\end{equation}
for some $k \geq 0$. First, we consider the derivative with respect to accelerometer inputs. Since the accelerometer can be excited along three orthogonal axes, we calculate the derivative along the individual axes: $\mathbf{a}_m = [\mathbf{1}_x a_x + \mathbf{1}_y a_y + \mathbf{1}_z a_z]$, where $\mathbf{1}_x = [1,0,0]^T$, $\mathbf{1}_y=[0,1,0]^T$ and $\mathbf{1}_z=[0,0,1]^T$, represent the axes of acceleration in the IMU body frame. The derivative of \eqref{eqn:meas_model}, when the accelerometer axes are excited is given by:
\begin{align*}
\dfrac{\partial h \left( \mathbf{p}^W_i, \mathbf{x} (t_r) \right)}{\partial a_x(t_r - t_d)} &= \dfrac{1}{2} \delta \mathbf{p}_i^T ~ \mathbf{R}^W_I(t_I) ~ \mathbf{1}_x ~t_d^2,\\
\dfrac{\partial h \left( \mathbf{p}^W_i, \mathbf{x} (t_r) \right)}{\partial a_y(t_r - t_d)} &= \dfrac{1}{2} \delta \mathbf{p}_i^T ~ \mathbf{R}^W_I(t_I) ~ \mathbf{1}_y ~t_d^2,\\
\dfrac{\partial h \left( \mathbf{p}^W_i, \mathbf{x} (t_r) \right)}{\partial a_z(t_r - t_d)} &= \dfrac{1}{2} \delta \mathbf{p}_i^T ~ \mathbf{R}^W_I(t_I) ~ \mathbf{1}_z ~t_d^2.
\end{align*}

where $\delta \mathbf{p}_i = \mathbf{p}^W_i - \mathbf{R}^W_I(t_I) \mathbf{p}^I_U - \mathbf{p}^W_I(t_I)$. The above derivatives are simultaneously zero if:
\begin{align*}
\delta \mathbf{p}_i^T \mathbf{R}^W_I(t_I) ~t_d^2  &= 0.
\end{align*}

For any non-zero time delay $t_d$ cannot be zero. Note that a valid rotation matrix $\mathbf{R}$ has determinant 1: its kernel space is trivial, $\{ \mathbf{R} \mathbf{v} = 0 \implies \mathbf{v} = \mathbf{0} \}$. Thus, $\delta \mathbf{p}_i^T \cdot \mathbf{R}^W_I(t_I) = 0$, only if $\delta \mathbf{p}_i = \mathbf{0}$, which requires the mobile UWB radio and the $i^{th}$ anchor to be co-located, practically this is not possible. Thus, if the mobile radio and the anchor are not co-located, then $t_d$ is locally identifiable when accelerometer axes are excited .

Now, we consider the derivative with respect to gyroscope inputs. Similar to the case above, the gyroscope can be excited along three orthogonal axes, we calculate the derivative along the individual axes: $\bm{\omega}_m = [\mathbf{1}_x \omega_x + \mathbf{1}_y \omega_y + \mathbf{1}_z \omega_z]$, where $\mathbf{1}_x = [1,0,0]^T$, $\mathbf{1}_y=[0,1,0]^T$ and $\mathbf{1}_z=[0,0,1]^T$, represent the axes of rotation in the IMU body frame. The derivative of \eqref{eqn:meas_model}, when $\omega_x$ is excited is given by:
\begin{align}
\dfrac{\partial h \left( \mathbf{p}^W_i, \mathbf{x} (t_r) \right)}{\partial \omega_x(t_r - t_d)} &= \dfrac{1}{2} \delta \mathbf{p}_i^T ~ \mathbf{R}^W_I(t_I) ~ [\mathbf{1}_x]_\times ~ \mathbf{p}^I_U ~ t_d,\\
&= \dfrac{1}{2} \left( {\mathbf{R}^W_I}^T(t_I)\delta \mathbf{p}_i \right)^T  ~ [\mathbf{1}_x]_\times ~ \mathbf{p}^I_U ~ t_d. \label{eqn:gyro_der}
\end{align}

Consider the following vector:
\begin{align}
{\mathbf{R}^W_I}^T(t_I)\delta \mathbf{p}_i &= {\mathbf{R}^W_I}^T(t_I) \left( \mathbf{p}^W_i - \mathbf{R}^W_I(t_I) \mathbf{p}^I_U - \mathbf{p}^W_I(t_I) \right), \nonumber \\
&= {\mathbf{R}^W_I}^T(t_I) \left( \mathbf{p}^W_i -  \mathbf{p}^W_I(t_I) \right) - \mathbf{p}^I_U, \label{eqn:imu_centric}
\end{align}

where use the fact for a rotation matrix $\mathbf{R}$, $\mathbf{R}^{-1} = \mathbf{R}^T$ and $\mathbf{R}^T \mathbf{R} = \mathbf{I}$. Equation \eqref{eqn:imu_centric} projects the measurement model into an IMU centeric frame. Using \eqref{eqn:imu_centric}  in \eqref{eqn:gyro_der} gives:
\begin{align}
\dfrac{\partial h \left( \mathbf{p}^W_i, \mathbf{x} (t_r) \right)}{\partial \omega_x(t_r - t_d)} &= \left( {\mathbf{R}^W_I}^T(t_I) \left( \mathbf{p}^W_i -  \mathbf{p}^W_I(t_I) \right) - \mathbf{p}^I_U \right)^T  ~ [\mathbf{1}_x]_\times ~ \mathbf{p}^I_U ~ t_d, \nonumber\\
&= \left( {\mathbf{R}^W_I}^T(t_I) \left( \mathbf{p}^W_i -  \mathbf{p}^W_I(t_I) \right) \right)^T [\mathbf{1}_x]_\times ~ \mathbf{p}^I_U ~ t_d - {\mathbf{p}^I_U}^T   ~ [\mathbf{1}_x]_\times ~ \mathbf{p}^I_U ~ t_d, \nonumber\\
&= \left( {\mathbf{R}^W_I}^T(t_I) \left( \mathbf{p}^W_i -  \mathbf{p}^W_I(t_I) \right) \right)^T [\mathbf{1}_x]_\times ~ \mathbf{p}^I_U ~ t_d + {\mathbf{p}^I_U}^T   ~ [\mathbf{p}^I_U]_\times ~ \mathbf{1}_x ~ t_d, \nonumber\\
&= \left( {\mathbf{R}^W_I}^T(t_I) \left( \mathbf{p}^W_i -  \mathbf{p}^W_I(t_I) \right) \right)^T [\mathbf{1}_x]_\times ~ \mathbf{p}^I_U ~ t_d, \nonumber\\
&= {\mathbf{p}^I_i}^T [\mathbf{1}_x]_\times ~ \mathbf{p}^I_U ~ t_d, \label{eqn:gyro_cond}
\end{align}
where we use the fact that for a skew-symmetric matrix $S$, $S^T = -S$ and $\mathbf{p}^I_i = {\mathbf{R}^W_I}^T(t_I) \left( \mathbf{p}^W_i -  \mathbf{p}^W_I(t_I) \right)$ is the position of the anchor expressed in IMU frame. The expression \eqref{eqn:gyro_cond} is zero only if the following conditions occur:
\begin{itemize}
\item[S1]: $t_d = 0$,  
\item[S2]: $\mathbf{p}^I_U = \mathbf{0}$, 
\item[S3]: $ \mathbf{p}^I_i = \mathbf{0}$,
\item[S4]: $\mathbf{1}_x \times \mathbf{p}^I_U  = \mathbf{0}$, or
\item[S5]: $\mathbf{1}_x \times \mathbf{p}^I_i = \mathbf{0}$.
\end{itemize}

For any non-zero time delay $t_d \neq 0$. Condition S2 occurs when the mobile radio and the IMU are co-located. Condition S3 requires the mobile radio and the $i^{th}$ anchor to be co-located. Thus, for $t_d$ to be identifiable, the mobile radio cannot be co-located with the IMU or the anchor. Condition S4 occurs when the axis of rotation is aligned with vector from the IMU to the mobile radio. Similary, condition S5 occurs when the axis of rotation is aligned with the vector from the IMU to the $i^{th}$ anchor. A similar analysis for $\omega_y$ and $\omega_z$ yields:
\begin{itemize}
\item[S6]: $\mathbf{1}_y \times \mathbf{p}^I_U  = \mathbf{0}$, 
\item[S7]: $\mathbf{1}_y \times \mathbf{p}^I_i = \mathbf{0}$,
\item[S8]: $\mathbf{1}_z \times \mathbf{p}^I_U  = \mathbf{0}$,  
\item[S9]: $\mathbf{1}_z \times \mathbf{p}^I_i = \mathbf{0}$.
\end{itemize}

Note that S4, S6, and S8 cannot be zero simultaneously, i.e. if $\mathbf{p}^I_U$ is aligned with $\mathbf{1}_x$, then it cannot be aligned with $\mathbf{1}_y$ which is orthogonal to $\mathbf{1}_x$. Similarly, S5, S7, and S9 cannot be zero simultaneously. However, it might be the case that $\omega_x$ aligns with $\mathbf{p}^I_i$ and $\omega_y$ aligns with $\mathbf{p}^I_U$. Thus, excitation of atleast one of $[\omega_x$, $\omega_y$, $\omega_z]^T$ is a necessary condition and excitation of all three is a sufficient condition for identifiability of $t_d$. With multiple non-collinear anchors, excitation of two of $\omega_x$, $\omega_y$, or $\omega_z$ is sufficient. A requirement for local identifiability of $t_d$ is that a change in the input causes a change in the measured range. Condition T3 in Theorem \ref{lem:td_cond_1} reflects the fact that if the mobile radio and the IMU are co-located, then for pure rotational motion the measured range is constant.
\end{proof}

\subsection{Observability of a tightly-coupled UWB-IMU system}

To study the observability of the system outlined in (\ref{eqn:non_linear_model_1})-(\ref{eqn:non_linear_model_3}) we use tools from differential geometry. Specifically, we use the concept of \textit{local weak observability} \cite{Hermann1977}. The analysis outlined in \cite{Hermann1977} considers the case of noise free autonomous systems without any delay in the observation model (excluding (\ref{eqn:non_linear_model_4})).

\subsubsection{Differential geometry and observability of nonlinear systems} \label{sec:obs_ana}
Given a vector field $\bm{f}: M \rightarrow TM$, where $TM \in \mathbb{R}^n$ is the tangent bundle of the smooth manifold $M$ of dimension $n$, the Lie derivative of  a smooth function $h: \mathbb{R}^n \rightarrow \mathbb{R}$ along the vector field $\bm{f}$ at $ \mathbf{x} \in M$ is:
\begin{equation}
	L_{\bm{f}} h(\mathbf{x}) = \dfrac{\partial{{h}(\mathbf{x})}}{\partial{\mathbf{x}}} \bm{f}(\mathbf{x}).
	\label{eqn:lie_derivative}
\end{equation}

The corresponding gradient vector is:
\begin{equation}
	\nabla {L_{\bm{f}} h(\mathbf{x})} = \begin{bmatrix}
	\dfrac{\partial{{L_{\bm{f}}{h}(\mathbf{x})}}}{\partial{x_1}},...,
	\dfrac{\partial{{L_{\bm{f}}{h}(\mathbf{x})}}}{\partial{x_n}},  
\end{bmatrix}.
\end{equation}
Higher-order Lie derivatives are defined recursively as:
\begin{align*}
	L^k_{\bm{f}}{h}(\mathbf{x}) = \nabla L^{k-1}_{\bm{f}}{h}(\mathbf{x}) \bm{f}(\mathbf{x}), 
\end{align*}
with the zeroth-order Lie derviative being $L^0_{\bm{f}}{h}(\mathbf{x}) = {h}(\mathbf{x})$.

We consider causal nonlinear systems that are affine in the control input: 
\begin{equation}
	\textbf{S} : 
	\begin{cases}
	\dot{\mathbf{x}} = \bm{f}_0(\mathbf{x}) + \sum\limits_{i} \bm{f}_i(\mathbf{x})u_i,\\
	y = {h}(\mathbf{x}),
	\end{cases}
	\label{eqn:control-affine-form}
\end{equation}

\noindent where $u \in \Gamma \subset \mathbb{R}^m$ is the input and $y \in \mathbb{R}$ is the output and $\mathbf{x} \in M$ is the state. The functions $\bm{f}_0$ and $\bm{f}_i$ are assumed to be smooth $C^{\infty}$. In this paper we use the concept of \textit{local weak observablility}. The system \textbf{S} in (\ref{eqn:control-affine-form}) is locally weakly observable at $\mathbf{x}_0$ if there exists an open neighbourhood $B$ of $\mathbf{x}_0$ in $M$ such that for every open neighbourhood $V$ of $\mathbf{x}_0$ contained in $B$, the only point \textit{indistinguishable} from $\mathbf{x}_0$ is $\mathbf{x}_0$. Let $\mathcal{O}$ denote the matrix formed by stacking the gradients of Lie derivatives of the observation function ${h}$. The system \textbf{S} is said to satisfy the \textit{observability rank condition} if $\mathcal{O}$ has full column rank. Furthermore, \textbf{S} is said to be locally weakly observable if it satisfies the observability rank condition.

\subsubsection{Observability analysis}
The system dynamics (\ref{eqn:non_linear_model_1})-(\ref{eqn:non_linear_model_3}) are rearranged to have a \emph{control affine}\cite{Hermann1977} form:
\begin{align}
	\dot{\widetilde{\mathbf{x}}} =
	\underbrace{\begin{bmatrix}[1.2]
		\mathbf{v}^W_I\\
		-\mathbf{R}^W_I\mathbf{b}_a - \mathbf{g}^W\\
		-\frac{1}{2}\mathbf{\Xi}\{\mathbf{q}^W_I\}\mathbf{b}_{\omega}\\
		\mathbf{0}_{3}\\
		\mathbf{0}_{3}\\
		\mathbf{0}_{3}\\
		\end{bmatrix}}_{f_0} +
	\underbrace{\begin{bmatrix}[1.2]
		\mathbf{0}_{3 \times 3}\\
		\mathbf{R}^W_I\\
		\mathbf{0}_{3 \times 4}\\
		\mathbf{0}_{3 \times 3}\\
		\mathbf{0}_{3 \times 3}\\
		\mathbf{0}_{3 \times 3}\\
	\end{bmatrix}}_{f_1} \mathbf{a}_m 
	+ \underbrace{\begin{bmatrix}[1.2]
		\mathbf{0}_{3 \times 3}\\
		\mathbf{0}_{3 \times 3}\\
		\frac{1}{2}\mathbf{\Xi}\{\mathbf{q}^W_I\}\\
		\mathbf{0}_{3 \times 3}\\
		\mathbf{0}_{3 \times 3}\\
		\mathbf{0}_{3 \times 3}\\
	\end{bmatrix}}_{f_2} \bm{\omega}_m,
	\label{eqn:system_dynamics_affine}
\end{align} 
where $\widetilde{\mathbf{x}} = (\mathbf{p}^W_I, \mathbf{v}^W_I, \mathbf{q}^W_I, \mathbf{b}_a, \mathbf{b}_\omega, \mathbf{p}^I_U)$ is the system state excluding the temporal offset $t_d$ and   
\begin{align*}
	\mathbf{\Xi}\{\mathbf{q}^W_I\} &= 
	\begin{bmatrix*}[r]
	  -q_1, & -q_2, & -q_3\\
       q_0, & -q_3, &  q_2\\
       q_3, &  q_0, & -q_1\\
      -q_2, &  q_1, &  q_0
      \end{bmatrix*}.
\end{align*}

Without loss of generality, the observation model (\ref{eqn:meas_model}) can be rewritten as:
\begin{equation}
	h(\mathbf{p}^W_i, \widetilde{\mathbf{x}}) = \dfrac{1}{2} \| \mathbf{p}^W_i - \mathbf{R}^W_I \mathbf{p}^I_U - \mathbf{p}^W_I \|^2_2,
	\nonumber
\end{equation}

Measurements from a single anchor are not sufficient to constraint the entire state. Hence, we consider measurements to three anchors $\mathbf{p}^W_a = [\mathbf{p}^W_i,~\mathbf{p}^W_j,~\mathbf{p}^W_k]$:
\begin{equation}
	\mathbf{h}(\mathbf{p}^W_a, \widetilde{\mathbf{x}}) = \dfrac{1}{2}\begin{bmatrix}[1.2]
	\|\mathbf{p}^W_i - \mathbf{R}^W_I \mathbf{p}^I_U - \mathbf{p}^W_I \|^2_2\\
	\|\mathbf{p}^W_j - \mathbf{R}^W_I \mathbf{p}^I_U - \mathbf{p}^W_I \|^2_2\\
	\|\mathbf{p}^W_k - \mathbf{R}^W_I \mathbf{p}^I_U - \mathbf{p}^W_I \|^2_2
	\end{bmatrix}. 
	\label{eqn:mul_anchor_meas_eq}
\end{equation}

We show that the system (\ref{eqn:system_dynamics_affine})-(\ref{eqn:mul_anchor_meas_eq}) is observable by employing the \textit{observability rank condition} mentioned above. The system (\ref{eqn:system_dynamics_affine})-(\ref{eqn:mul_anchor_meas_eq}) is said to be locally weakly observable if the corresponding observability matrix $\mathcal{O}$ has full column rank, which in this case is 19, corresponding to the size of the state (\ref{eqn:state}).
To prove this, we construct the matrix $\mathcal{O}$ by stacking higher-order Lie derivatives of the vector-valued function (\ref{eqn:mul_anchor_meas_eq}) along the vector field (\ref{eqn:system_dynamics_affine}). Block Gaussian elimination is then used to identify conditions under which $\mathcal{O}$ has full column rank.

\subsubsection{Observability Matrix Construction}
The observability matrix $\mathcal{O}$ is constructed by taking higher-order Lie derivatives of the measurement function along the system dynamics \cite{Hermann1977}.\\ \\
\noindent \textbf{Zeroth order Lie derivatives}\\
1. The zeroth order Lie derivative of $\mathbf{h}$ is:
\begin{equation}
	L^0 \mathbf{h}(\widetilde{\mathbf{x}}) = \mathbf{h}(\widetilde{\mathbf{x}}).
\end{equation}
The corresponding gradient is:
\begin{align}
	\nabla L^0{\mathbf{h}(\widetilde{\mathbf{x}})} = 
	\left[ 
		\begin{matrix}
		 -{\delta \mathbf{p}_{ijk}} & \mathbf{0}_{3 \times 3} & -{\delta \mathbf{p}_{ijk}} \mathbf{F}_0 & \mathbf{0}_{3 \times 3} & \mathbf{0}_{3 \times 3} & -{\delta \mathbf{p}_{ijk}} \mathbf{R}^W_I
  		\end{matrix}
  	\right],
\end{align}
where $\mathbf{F}_0$ is a 3x4 matrix and is given by,
\begin{align}
	\mathbf{F}_0 &= \dfrac{\partial \left( \mathbf{R}^W_I \mathbf{p}^I_U \right)}{\partial{\mathbf{q}^W_I}},	
\end{align}
and $\delta \mathbf{p}_{ijk} = [\delta \mathbf{p}_i^T, \delta \mathbf{p}_j^T, \delta \mathbf{p}_k^T]^T$ is a matrix of residuals with:
\begin{align*} 
	\delta \mathbf{p}_i &= \mathbf{p}^W_i - \mathbf{R}^W_I \mathbf{p}^I_U - \mathbf{p}^W_I,\\
	\delta \mathbf{p}_j &= \mathbf{p}^W_j - \mathbf{R}^W_I \mathbf{p}^I_U - \mathbf{p}^W_I,\\
	\delta \mathbf{p}_k &= \mathbf{p}^W_k - \mathbf{R}^W_I \mathbf{p}^I_U - \mathbf{p}^W_I.
\end{align*}

\noindent \textbf{First order Lie derivatives}\\
1. The first order Lie derivative of $\mathbf{h}$ along $f_0$:
\begin{align}
	L^1_{f_0} \mathbf{h}(\widetilde{\mathbf{x}}) &= \nabla L^0{\mathbf{h}(\widetilde{\mathbf{x}})} \cdot f_0, \nonumber \\
	&= \left[
	   		\begin{matrix}
	   	-\delta \mathbf{p}_{ijk} \mathbf{v}^W_I + \dfrac{1}{2} \delta \mathbf{p}_{ijk} \mathbf{F}_0 \bm{\Xi}\{\mathbf{q}^W_I\} \mathbf{b}_{\omega}
	   	\end{matrix}
	   \right].
\end{align}
The gradient of $L^1_{f_0} \mathbf{h}(\widetilde{\mathbf{x}})$ is:
\begin{align}
	\nabla L^1_{f_0} \mathbf{h}(\widetilde{\mathbf{x}}) &= 
	\left[ 
		 \begin{matrix} 
		 \mathbf{F}_1 & -\delta \mathbf{p}_{ijk} & \mathbf{F}_2 & \mathbf{0}_{3 \times 3} & \delta \mathbf{p}_{ijk} \mathbf{F}_3 & \mathbf{F}_4 
		 \end{matrix}
	\right],
\end{align} 
\noindent with $\mathbf{F}_1$, $\mathbf{F}_2$, $\mathbf{F}_3$ and $\mathbf{F}_4$ defined as:
\begin{align}
	&\mathbf{F}_1 = 
		\begin{bmatrix}
		\left( \mathbf{v}^W_I - \frac{1}{2} \mathbf{F}_0 \bm{\Xi}\{\mathbf{q}^W_I\} \mathbf{b}_{\omega} \right) ^T\\
		\left( \mathbf{v}^W_I - \frac{1}{2} \mathbf{F}_0 \bm{\Xi}\{\mathbf{q}^W_I\} \mathbf{b}_{\omega} \right) ^T\\
		\left( \mathbf{v}^W_I - \frac{1}{2} \mathbf{F}_0 \bm{\Xi}\{\mathbf{q}^W_I\} \mathbf{b}_{\omega} \right)^T
		\end{bmatrix},
	&\mathbf{F}_2 = \dfrac{\partial\{-\delta \mathbf{p}_{ijk} \mathbf{F}_1^T\}}{\partial{\mathbf{q}^W_I}},\\
	&\mathbf{F}_3 =  \dfrac{1}{2}\mathbf{F}_0 \bm{\Xi}\{\mathbf{q}^W_I\},
	&\mathbf{F}_4 = \dfrac{\partial\{-\delta \mathbf{p}_{ijk} \mathbf{F}_1^T\}}{\partial{\mathbf{p}^I_U}}.
\end{align}

\noindent \textbf{Second order Lie Derivatives}\\
1. $L^1_{f_0} \mathbf{h}(\widetilde{\mathbf{x}})$ along $f_1$:
\begin{align}
 L^1_{f_1} L^1_{f_0} \mathbf{h}(\widetilde{\mathbf{x}}) &= \nabla  L^1_{f_0} \mathbf{h}(\widetilde{\mathbf{x}}) \cdot f_1, \\
&= - \delta \mathbf{p}_{ijk} \mathbf{R}^W_I.
\end{align}

Note that this is a 3x3 matrix with the first column being influenced by measured acceleration along body x-axis ($a_x$). Similarly the second and third column are influenced by $a_y$ and $a_z$ respectively. Thus, the individual columns are stacked to form a 9x1 vector and the corresponding gradients are computed:
\begin{align}
\nabla L^1_{f_1} L^1_{f_0} \mathbf{h}(\widetilde{\mathbf{x}}) &= \begin{bmatrix}
\nabla L^1_{f_1} L^1_{f_0} \mathbf{h}(\widetilde{\mathbf{x}})[:,1] \\
\nabla L^1_{f_1} L^1_{f_0} \mathbf{h}(\widetilde{\mathbf{x}})[:,2] \\
\nabla L^1_{f_1} L^1_{f_0} \mathbf{h}(\widetilde{\mathbf{x}})[:,3]
\end{bmatrix},\\
&= \begin{bmatrix}
\mathbf{F}_5 & \mathbf{0}_{9 \times 3} & \mathbf{F}_6 & \mathbf{0}_{9 \times 3} & \mathbf{0}_{9 \times 3} & \mathbf{F}_{7}
\end{bmatrix}.
\end{align}

\noindent 2. $L^1_{f_0} \mathbf{h}(\widetilde{\mathbf{x}})$ along $f_0$:
\begin{align}
	L^2_{f_0} \mathbf{h}(\widetilde{\mathbf{x}}) &= \nabla L^1_{f_0}{\mathbf{h}(\widetilde{\mathbf{x}})} \cdot f_0, \nonumber \\
		&= 
	   	\bigg[\mathbf{F}_1\mathbf{v}^W_I + \delta{\mathbf{p}_{ijk}} \left(\mathbf{R}^W_I\mathbf{b}_a + \mathbf{g}^W \right) \nonumber 
	   	  - \dfrac{1}{2} \mathbf{F}_2 \bm{\Xi} \{ \mathbf{q}^W_I \} \mathbf{b}_{\omega} \bigg].
\end{align}

The gradient of $L^2_{f_0} \mathbf{h}(\widetilde{\mathbf{x}})$ is:
\begin{align}
	\nabla L^2_{f_0} \mathbf{h}(\widetilde{\mathbf{x}}) &= 
	\begin{bmatrix}
		\mathbf{F}_8 & \mathbf{F}_9 & \mathbf{F}_{10} & \delta \mathbf{p}_{ijk}\mathbf{R}^W_I & \mathbf{F}_{11} & \mathbf{F}_{12} \end{bmatrix}.
\end{align}

\noindent 3. $L^0 \mathbf{h}(\widetilde{\mathbf{x}})$ along $f_2$:
\begin{align}
\nabla L^1_{f_{2}} L^{0} \mathbf{h}(\widetilde{\mathbf{x}}) &= -\delta \mathbf{p}_{ijk} \mathbf{F}_0 \mathbf{\Xi}\{ \mathbf{q}^W_I \}.
\end{align}

Note that this is a 3x3 matrix with the first column being influenced by angular velocity along body x-axis ($\omega_x$). Similarly the second and third column are influenced by $\omega_y$ and $\omega_z$ respectively. Thus, the individual columns are stacked to form a 9x1 vector and the corresponding gradients are computed:
\begin{align}
\nabla L^1_{f_2} L^0 \mathbf{h}(\widetilde{\mathbf{x}}) &= \begin{bmatrix}
\nabla L^1_{f_2} L^0 \mathbf{h}(\widetilde{\mathbf{x}})[:,1] \\
\nabla L^1_{f_2} L^0 \mathbf{h}(\widetilde{\mathbf{x}})[:,2] \\
\nabla L^1_{f_2} L^0 \mathbf{h}(\widetilde{\mathbf{x}})[:,3]
\end{bmatrix},\\
&= \begin{bmatrix}
\mathbf{F}_{13} & \mathbf{0}_{9 \times 3} & \mathbf{F}_{14} & \mathbf{0}_{9 \times 3} & \mathbf{0}_{9 \times 3} & \mathbf{F}_{15}
\end{bmatrix}.
\end{align}

\noindent \textbf{Third order Lie derivatives}\\

\noindent 1. $L^1_{f_1} L^1_{f_0} \mathbf{h}(\widetilde{\mathbf{x}})$ along $f_0$:
\begin{align}
L^1_{f_0} L^1_{f_1} L^1_{f_0} \mathbf{h}(\widetilde{\mathbf{x}}) &= \nabla L^1_{f_1} L^1_{f_0} \mathbf{h}(\widetilde{\mathbf{x}}) \cdot f_0, \\
&= \mathbf{F}_5 \mathbf{v}^W_I - \dfrac{1}{2} \mathbf{F}_6 \mathbf{\Xi} \{ \mathbf{q}^W_I \} \mathbf{b}_w.
\end{align}

The corresponding gradient is:
\begin{align}
\nabla L^1_{f_0} L^1_{f_1} L^1_{f_0} \mathbf{h}(\widetilde{\mathbf{x}}) &= \begin{bmatrix}
\mathbf{F}_{16} & \mathbf{F}_5 & \mathbf{F}_{17} & \mathbf{0}_{9 \times 3} & \mathbf{F}_{18} & \mathbf{F}_{19}
\end{bmatrix}.
\end{align}

The observability matrix $\mathcal{O}$ is then constructed by stacking the Lie derivatives computed so far:
\begin{align*}
  \mathcal{O} = \begin{bmatrix}
	\nabla L^0 \mathbf{h}\\
	\nabla L^1_{f_0} \mathbf{h}\\
	\nabla L^2_{f_0} \mathbf{h}\\	
	\nabla L^1_{f_1}L^1_{f_0} \mathbf{h}\\
	\nabla L^1_{f_{2}} L^{0} \mathbf{h}\\
	\nabla L^1_{f_0} L^1_{f_1} L^1_{f_0} \mathbf{h}
	\end{bmatrix} =
    \begin{bmatrix}
    	-{\delta \mathbf{p}_{ijk}} & \mathbf{0}_{3 \times 3} & -{\delta \mathbf{p}_{ijk}} \mathbf{F}_0 & \mathbf{0}_{3 \times 3} & \mathbf{0}_{3 \times 3} & -{\delta \mathbf{p}_{ijk}} \mathbf{R}^W_I\\
		\mathbf{F}_1 & -\delta \mathbf{p}_{ijk} &  \mathbf{F}_2 &  \mathbf{0}_{3 \times 3} &  \delta \mathbf{p}_{ijk} \mathbf{F}_3 &  \mathbf{F}_4\\
		\mathbf{F}_8 &  \mathbf{F}_9 &  \mathbf{F}_{10} &  {\delta \mathbf{p}_{ijk}} \mathbf{R}^W_I &  \mathbf{F}_{11} &  \mathbf{F}_{12} \\
		\mathbf{F}_5 & \mathbf{0}_{9 \times 3} & \mathbf{F}_6 & \mathbf{0}_{9 \times 3} & \mathbf{0}_{9 \times 3} & \mathbf{F}_7\\
 		\mathbf{F}_{13} & \mathbf{0}_{9 \times 3} & \mathbf{F}_{14} & \mathbf{0}_{9 \times 3} & \mathbf{0}_{9 \times 3} & \mathbf{F}_{15} \\
		\mathbf{F}_{16} & \mathbf{F}_5 & \mathbf{F}_{17} & \mathbf{0}_{9 \times 3} & \mathbf{F}_{18} & \mathbf{F}_{19}		
\end{bmatrix}
\end{align*}

\subsubsection{Proof of local weak observability}
\begin{theorem} 
The observability matrix $\mathcal{O}$ associated with motion model (\ref{eqn:system_dynamics_affine}) and observation model (\ref{eqn:mul_anchor_meas_eq}) has full column rank if:
\begin{itemize}
\itemsep0em
\item[(C1)] at least three non-collinear anchors are available,
\item[(C2)] the mobile radio is non-coplanar with the three non-collinear anchors,
\item[(C3)] all three of $a_x, a_y$ and $a_z$ are excited and
\item[(C4)] all three of $\omega_x, \omega_y$ and $\omega_z$ are excited.
\end{itemize}
\label{thm:obs_cond}
\end{theorem}

\begin{proof}
We use block Gaussian elimination to prove that the matrix $\mathcal{O}$ has full column rank. The process of Gaussian Elimination followed here exploits the fact that \textit{both row and column operations} can be used interchangeably as the rank of the matrix is unaffected by such operations \cite{Lang}.

\subsubsection*{Step I}
The following two columns operations are performed:\\
Column 3 $\rightarrow$ Column 1 * $\mathbf{F}_0$ - Column 3\\
Column 6 $\rightarrow$ Column 1 * $\mathbf{R}^W_I$ - Column 6\\
Column 5 $\rightarrow$ Column 2 * $\mathbf{F}_3$ + Column 5.

\begin{align*}
  \mathcal{O} = 
    \begin{bmatrix}
    	-{\delta \mathbf{p}_{ijk}} & \mathbf{0}_{3 \times 3} & \mathbf{0}_{3 \times 4} & \mathbf{0}_{3 \times 3} & \mathbf{0}_{3 \times 3} & \mathbf{0}_{3 \times 3}\\
		\mathbf{F}_1 & -\delta \mathbf{p}_{ijk} &  \mathbf{F}_1 \mathbf{F}_0 - \mathbf{F}_2  &  \mathbf{0}_{3 \times 3} & \mathbf{0}_{3 \times 3}  &  \mathbf{F}_1  \mathbf{R}^W_I - \mathbf{F}_4\\
	    \mathbf{F}_8 &  \mathbf{F}_9 & \mathbf{F}_8 \mathbf{F}_0 - \mathbf{F}_{10} &  {\delta \mathbf{p}_{ijk}} \mathbf{R}^W_I &  \mathbf{F}_9 \mathbf{F}_3 - \mathbf{F}_{11} & \mathbf{F}_8  \mathbf{R}^W_I - \mathbf{F}_{12}\\
		\mathbf{F}_5 & \mathbf{0}_{9 \times 3} & \mathbf{F}_5  \mathbf{F}_0 - \mathbf{F}_6 & \mathbf{0}_{9 \times 3} & \mathbf{0}_{9 \times 3} & \mathbf{F}_5  \mathbf{R}^W_I - \mathbf{F}_7\\
		\mathbf{F}_{13} & \mathbf{0}_{9 \times 3}  & \mathbf{F}_{13}  \mathbf{F}_0 - \mathbf{F}_{14} & \mathbf{0}_{9 \times 3} & \mathbf{0}_{9 \times 3} & \mathbf{F}_{13}  \mathbf{R}^W_I - \mathbf{F}_{15}\\
		\mathbf{F}_{16} & \mathbf{F}_5 & \mathbf{F}_{16} \mathbf{F}_0 - \mathbf{F}_{17} & \mathbf{0}_{9 \times 3} & \mathbf{F}_5 \mathbf{F}_3 + \mathbf{F}_{18} & \mathbf{F}_{16}  \mathbf{R}^W_I  - \mathbf{F}_{19}
    \end{bmatrix}
\end{align*}

\subsubsection*{Step II}
Column 4 has full column rank contingent on ${\delta \mathbf{p}_{ijk}} \mathbf{R}^W_I$ being full rank. A matrix is full rank if its determinant is non-zero. Note that the determinant of a product of matrices is equal to the product of their individual determinants: for matrices $A$ and $B$, $|\mathbf{A}\mathbf{B} | = |\mathbf{A}||\mathbf{B}|$, where $| \cdot |$ denotes matrix determinant. Using this identity: 
\begin{align*}
|{\delta \mathbf{p}_{ijk}} \mathbf{R}^W_I| &= |{\delta \mathbf{p}_{ijk}}| |\mathbf{R}^W_I|
\end{align*}

The determinant of any valid rotation matrix is 1 and so $|\mathbf{R}^W_I| = 1$. The matrix $\delta \mathbf{p}_{ijk}$ has full column rank when the conditions mentioned in Lemma \ref{lem1} are satisfied. Under the conditions of Lemma \ref{lem1} column 4 and row 4 can be eliminated.
\begin{align*}
  \mathcal{O} = 
    \begin{bmatrix}
    	-{\delta \mathbf{p}_{ijk}} & \mathbf{0}_{3 \times 3} & \mathbf{0}_{3 \times 4} & \mathbf{0}_{3 \times 3} & \mathbf{0}_{3 \times 3} & \mathbf{0}_{3 \times 3}\\
		\mathbf{F}_1 & -\delta \mathbf{p}_{ijk} &  \mathbf{F}_1 \mathbf{F}_0 - \mathbf{F}_2  &  \mathbf{0}_{3 \times 3} & \mathbf{0}_{3 \times 3}  &  \mathbf{F}_1  \mathbf{R}^W_I - \mathbf{F}_4\\
	    \mathbf{0}_{3 \times 3} &  \mathbf{0}_{3 \times 3} & \mathbf{0}_{3 \times 4} &  \mathbf{I}_{3 \times 3} &  \mathbf{0}_{3 \times 3} & \mathbf{0}_{3 \times 3}\\
		\mathbf{F}_5 & \mathbf{0}_{9 \times 3} & \mathbf{F}_5  \mathbf{F}_0 - \mathbf{F}_6 & \mathbf{0}_{9 \times 3} & \mathbf{0}_{9 \times 3} & \mathbf{F}_5  \mathbf{R}^W_I - \mathbf{F}_7\\
		\mathbf{F}_{13} & \mathbf{0}_{9 \times 3}  & \mathbf{F}_{13}  \mathbf{F}_0 - \mathbf{F}_{14} & \mathbf{0}_{9 \times 3} & \mathbf{0}_{9 \times 3} & \mathbf{F}_{13}  \mathbf{R}^W_I - \mathbf{F}_{15}\\
		\mathbf{F}_{16} & \mathbf{F}_5 & \mathbf{F}_{16} \mathbf{F}_0 - \mathbf{F}_{17} & \mathbf{0}_{9 \times 3} & \mathbf{F}_5 \mathbf{F}_3 + \mathbf{F}_{18} & \mathbf{F}_{16}  \mathbf{R}^W_I  - \mathbf{F}_{19}
    \end{bmatrix}
\end{align*}

\subsubsection*{Step III}
Similarly, ${\delta \mathbf{p}_{ijk}}$ has full column rank and column 1 entries can be eliminated:
\begin{align*}
  \mathcal{O} = 
    \begin{bmatrix}
    	\mathbf{I}_{3 \times 3} & \mathbf{0}_{3 \times 3} & \mathbf{0}_{3 \times 4} & \mathbf{0}_{3 \times 3} & \mathbf{0}_{3 \times 3} & \mathbf{0}_{3 \times 3}\\
		\mathbf{0}_{3 \times 3} & -\delta \mathbf{p}_{ijk} &  \mathbf{F}_1 \mathbf{F}_0 - \mathbf{F}_2  &  \mathbf{0}_{3 \times 3} & \mathbf{0}_{3 \times 3}  &  \mathbf{F}_1  \mathbf{R}^W_I - \mathbf{F}_4\\
	    \mathbf{0}_{3 \times 3} &  \mathbf{0}_{3 \times 3} & \mathbf{0}_{3 \times 4} &  \mathbf{I}_{3 \times 3} &  \mathbf{0}_{3 \times 3} & \mathbf{0}_{3 \times 3}\\
		\mathbf{0}_{9 \times 3} & \mathbf{0}_{9 \times 3} & \mathbf{F}_5  \mathbf{F}_0 - \mathbf{F}_6 & \mathbf{0}_{9 \times 3} & \mathbf{0}_{9 \times 3} & \mathbf{F}_5  \mathbf{R}^W_I - \mathbf{F}_7\\
		\mathbf{0}_{9 \times 3} & \mathbf{0}_{9 \times 3}  & \mathbf{F}_{13}  \mathbf{F}_0 - \mathbf{F}_{14} & \mathbf{0}_{9 \times 3} & \mathbf{0}_{9 \times 3} & \mathbf{F}_{13}  \mathbf{R}^W_I - \mathbf{F}_{15}\\
		\mathbf{0}_{9 \times 3} & \mathbf{F}_5 & \mathbf{F}_{16} \mathbf{F}_0 - \mathbf{F}_{17} & \mathbf{0}_{9 \times 3} & \mathbf{F}_5 \mathbf{F}_3 + \mathbf{F}_{18} & \mathbf{F}_{16}  \mathbf{R}^W_I  - \mathbf{F}_{19}
    \end{bmatrix}
\end{align*}

\subsection*{Step IV}
Note that $\mathbf{F}_5  \mathbf{R}^W_I - \mathbf{F}_7$ is identically zero. Lemma \ref{lem2} gives the conditions under which $\mathbf{F}_5  \mathbf{F}_0 - \mathbf{F}_6$ has full column rank. Column 3 entries can be eliminated safely:
\begin{align*}
  \mathcal{O} = 
    \begin{bmatrix}
    	\mathbf{I}_{3 \times 3} & \mathbf{0}_{3 \times 3} & \mathbf{0}_{3 \times 4} & \mathbf{0}_{3 \times 3} & \mathbf{0}_{3 \times 3} & \mathbf{0}_{3 \times 3}\\
		\mathbf{0}_{3 \times 3} & -\delta \mathbf{p}_{ijk} &  \mathbf{0}_{3 \times 4}  &  \mathbf{0}_{3 \times 3} & \mathbf{0}_{3 \times 3}  &  \mathbf{F}_1  \mathbf{R}^W_I - \mathbf{F}_4\\
	    \mathbf{0}_{3 \times 3} &  \mathbf{0}_{3 \times 3} & \mathbf{0}_{3 \times 4} &  \mathbf{I}_{3 \times 3} &  \mathbf{0}_{3 \times 3} & \mathbf{0}_{3 \times 3}\\
		\mathbf{0}_{4 \times 3} & \mathbf{0}_{4 \times 3} & \mathbf{I}_{4 \times 4} & \mathbf{0}_{4 \times 3} & \mathbf{0}_{4 \times 3} & \mathbf{0}_{4 \times 3}\\
		\mathbf{0}_{9 \times 3} & \mathbf{0}_{9 \times 3}  & \mathbf{0}_{9 \times 4} & \mathbf{0}_{9 \times 3} & \mathbf{0}_{9 \times 3} & \mathbf{F}_{13}  \mathbf{R}^W_I - \mathbf{F}_{15}\\
		\mathbf{0}_{9 \times 3} & \mathbf{F}_5 & \mathbf{0}_{9 \times 4} & \mathbf{0}_{9 \times 3} & \mathbf{F}_5 \mathbf{F}_3 + \mathbf{F}_{18} & \mathbf{F}_{16}  \mathbf{R}^W_I  - \mathbf{F}_{19}
    \end{bmatrix}
\end{align*}

\subsection*{Step V}
$\mathbf{F}_{13}  \mathbf{R}^W_I - \mathbf{F}_{15}$ has full column rank under the conditions of Lemma \ref{lem3}. The entries of column and row 6 can be eliminated: 
\begin{align*}
  \mathcal{O} = 
    \begin{bmatrix}
    	\mathbf{I}_{3 \times 3} & \mathbf{0}_{3 \times 3} & \mathbf{0}_{3 \times 4} & \mathbf{0}_{3 \times 3} & \mathbf{0}_{3 \times 3} & \mathbf{0}_{3 \times 3}\\
		\mathbf{0}_{3 \times 3} & -\delta \mathbf{p}_{ijk} &  \mathbf{0}_{3 \times 4}  &  \mathbf{0}_{3 \times 3} & \mathbf{0}_{3 \times 3}  &  \mathbf{0}_{3 \times 3}\\
	    \mathbf{0}_{3 \times 3} &  \mathbf{0}_{3 \times 3} & \mathbf{0}_{3 \times 4} &  \mathbf{I}_{3 \times 3} &  \mathbf{0}_{3 \times 3} & \mathbf{0}_{3 \times 3}\\
		\mathbf{0}_{4 \times 3} & \mathbf{0}_{4 \times 3} & \mathbf{I}_{4 \times 4} & \mathbf{0}_{4 \times 3} & \mathbf{0}_{4 \times 3} & \mathbf{0}_{4 \times 3}\\
		\mathbf{0}_{3 \times 3} & \mathbf{0}_{3 \times 3} & \mathbf{0}_{3 \times 4} & \mathbf{0}_{3 \times 3} & \mathbf{0}_{3 \times 3} & \mathbf{I}_{3 \times 3}\\
		\mathbf{0}_{9 \times 3} & \mathbf{F}_5 & \mathbf{0}_{9 \times 4} & \mathbf{0}_{9 \times 3} & \mathbf{F}_5 \mathbf{F}_3 + \mathbf{F}_{18} & \mathbf{0}_{9 \times 3}
    \end{bmatrix}
\end{align*}

\subsection*{Step VI}
Under the conditions outlined in Lemma \ref{lem1}, $\delta \mathbf{p}_{ijk}$ is full rank, all entries of column 2 can be eliminated:
\begin{align*}
  \mathcal{O} = 
    \begin{bmatrix}
    	\mathbf{I}_{3 \times 3} & \mathbf{0}_{3 \times 3} & \mathbf{0}_{3 \times 4} & \mathbf{0}_{3 \times 3} & \mathbf{0}_{3 \times 3} & \mathbf{0}_{3 \times 3}\\
		\mathbf{0}_{3 \times 3} & \mathbf{I}_{3 \times 3} &  \mathbf{0}_{3 \times 4}  &  \mathbf{0}_{3 \times 3} & \mathbf{0}_{3 \times 3}  &  \mathbf{0}_{3 \times 3}\\
	    \mathbf{0}_{3 \times 3} &  \mathbf{0}_{3 \times 3} & \mathbf{0}_{3 \times 4} &  \mathbf{I}_{3 \times 3} &  \mathbf{0}_{3 \times 3} & \mathbf{0}_{3 \times 3}\\
		\mathbf{0}_{4 \times 3} & \mathbf{0}_{4 \times 3} & \mathbf{I}_{4 \times 4} & \mathbf{0}_{4 \times 3} & \mathbf{0}_{4 \times 3} & \mathbf{0}_{4 \times 3}\\
		\mathbf{0}_{3 \times 3} & \mathbf{0}_{3 \times 3} & \mathbf{0}_{3 \times 4} & \mathbf{0}_{3 \times 3} & \mathbf{0}_{3 \times 3} & \mathbf{I}_{3 \times 3}\\
		\mathbf{0}_{9 \times 3} & \mathbf{0}_{9 \times 3} & \mathbf{0}_{9 \times 4} & \mathbf{0}_{9 \times 3} & \mathbf{F}_5 \mathbf{F}_3 + \mathbf{F}_{18} & \mathbf{0}_{9 \times 3}
    \end{bmatrix}
\end{align*}

\subsection*{Step VII}
Finally, Lemma \ref{lem4} elicits the conditoins under which $\mathbf{F}_5 \mathbf{F}_3 + \mathbf{F}_{18}$ has full column rank:
\begin{align*}
  \mathcal{O} = 
    \begin{bmatrix}
    	\mathbf{I}_{3 \times 3} & \mathbf{0}_{3 \times 3} & \mathbf{0}_{3 \times 4} & \mathbf{0}_{3 \times 3} & \mathbf{0}_{3 \times 3} & \mathbf{0}_{3 \times 3}\\
		\mathbf{0}_{3 \times 3} & \mathbf{I}_{3 \times 3} &  \mathbf{0}_{3 \times 4}  &  \mathbf{0}_{3 \times 3} & \mathbf{0}_{3 \times 3}  &  \mathbf{0}_{3 \times 3}\\
	    \mathbf{0}_{3 \times 3} &  \mathbf{0}_{3 \times 3} & \mathbf{0}_{3 \times 4} &  \mathbf{I}_{3 \times 3} &  \mathbf{0}_{3 \times 3} & \mathbf{0}_{3 \times 3}\\
		\mathbf{0}_{4 \times 3} & \mathbf{0}_{4 \times 3} & \mathbf{I}_{4 \times 4} & \mathbf{0}_{4 \times 3} & \mathbf{0}_{4 \times 3} & \mathbf{0}_{4 \times 3}\\
		\mathbf{0}_{3 \times 3} & \mathbf{0}_{3 \times 3} & \mathbf{0}_{3 \times 4} & \mathbf{0}_{3 \times 3} & \mathbf{0}_{3 \times 3} & \mathbf{I}_{3 \times 3}\\
		\mathbf{0}_{3 \times 3} & \mathbf{0}_{3 \times 3} & \mathbf{0}_{3 \times 4} & \mathbf{0}_{3 \times 3} & \mathbf{I}_{3 \times 3} & \mathbf{0}_{3 \times 3}
    \end{bmatrix}
\end{align*}
\end{proof}

The above proof shows that as long as the conditions outlined in Theorem \ref{thm:obs_cond} are satisfied, the state $\widetilde{\mathbf{x}}$ is locally weakly observable.

\newpage
\appendix
\section{Lemmas and their proofs}
\begin{lemma}: ${\delta \mathbf{p}_{ijk}}$ is full rank if:
\begin{itemize}
\item at least three non-collinear anchors $i,j$ and $k$ are available; and
\item the mobile radio is non-coplanar with the three non-collinear anchors.
\end{itemize}
\label{lem1}
\end{lemma}

\begin{proof}
The matrix ${\delta \mathbf{p}_{ijk}}$ is given by:
\begin{align*}
{\delta \mathbf{p}_{ijk}} &= \begin{bmatrix}
	\mathbf{p}^W_i - \mathbf{R}^W_I \mathbf{p}^I_U - \mathbf{p}^W_I \\
	\mathbf{p}^W_j - \mathbf{R}^W_I \mathbf{p}^I_U - \mathbf{p}^W_I \\
	\mathbf{p}^W_k - \mathbf{R}^W_I \mathbf{p}^I_U - \mathbf{p}^W_I 
	\end{bmatrix}\\
	&= \begin{bmatrix}
	\mathbf{p}^W_{ix} & \mathbf{p}^W_{iy} & \mathbf{p}^W_{iz} \\
	\mathbf{p}^W_{jx} & \mathbf{p}^W_{jy} & \mathbf{p}^W_{jz} \\
	\mathbf{p}^W_{kx} & \mathbf{p}^W_{ky} & \mathbf{p}^W_{kz}
	\end{bmatrix}
	\end{align*}
where $\mathbf{p}^W_{ix}$ denotes the difference between the x-coordinate of the $i^{th}$ anchor and the x-coordinate of the mobile radio expressed in the world frame, $W$. The determinant is:
\begin{align*}
|\delta \mathbf{p}_{ijk}| = p^W_{jy} p^W_{kx} (p^W_{Iz} + p^I_{Ux} (-2 q_0 q_2 + 2 q_1 q_3) + p^I_{Uy} (2 q_0 q_1 + 2 q_2 q_3) + p^I_{Uz} (q_0^2 - q_1^2 - q_2^2 + q_3^2))
\end{align*}

Note that $p^W_{jy}$ and $p^W_{kx}$ cannot be zero if the anchors are to be non-collinear. Thus, the determinant is zero if the third term vanishes. This term represents the  z-coordinate of the position of the mobile radio in the world frame. Thus, for the matrix $\delta  \mathbf{p}_{ijk}^T$ to have full rank, the position of mobile radio expressed in world frame cannot lie in the plane defined by the non-collinear anchors $\{ i,j,k \}$.
\end{proof}

\begin{lemma} For any unit-quaternion $\mathbf{q}$, the 9x4 matrix	$\mathbf{F}_5  \mathbf{F}_0 - \mathbf{F}_6$ has full column rank when at least two components of $\mathbf{a}_m = [a_x, a_y, a_z]^T$ are excited.
\label{lem2}
\end{lemma}

\begin{proof}
To show that $\mathbf{F}_5  \mathbf{F}_0 - \mathbf{F}_6$ has full column rank, it is sufficient to show that the determinants of any combination of 4 rows cannot vanish simultaneously. The mathematica script used to calculate these determinants is provided in appendix B.1. The following determinants are considered:
\begin{align*}
\text{det}(1,2,3,4) &= -16 p^W_{jy} p^W_{kx} (p^I_{Uz} + p^W_{Iz} + 2 p^I_{Uy} q_0 q_1 - 2 p^I_{Uz} q_1^2 - 2 p^I_{Ux} q_0 q_2 - 
   2 p^I_{Uz} q_2^2 + 2 p^I_{Ux} q_1 q_3 + 2 p^I_{Uy} q_2 q_3)\\ & \qquad (p^I_{Uz} + p^W_{Iz} - 2 p^W_{Iy} q_0 q_1 - 
   2 p^W_{Iz} q_1^2 + 2 p^W_{Ix} q_0 q_2 - 2 p^W_{Iz} q_2^2 + 2 p^W_{Ix} q_1 q_3 + 2 p^W_{Iy} q_2 q_3) \\
\text{det}(1,2,3,7) &= 16 p^W_{jy} p^W_{kx} (p^I_{Uz} + p^W_{Iz} + 2 p^I_{Uy} q_0 q_1 - 2 p^I_{Uz} q_1^2 - 2 p^I_{Ux} q_0 q_2 - 
   2 p^I_{Uz} q_2^2 + 2 p^I_{Ux} q_1 q_3 + 2 p^I_{Uy} q_2 q_3)\\ & \qquad (p^I_{Uy} + p^W_{Iy} + 2 p^W_{Iz} q_0 q_1 - 
   2 p^W_{Iy} q_1^2 + 2 p^W_{Ix} q_1 q_2 - 2 p^W_{Ix} q_0 q_3 + 2 p^W_{Iz} q_2 q_3 - 2 p^W_{Iy} q_3^2)\\
\text{det}(4,5,6,7) &= -16 p^W_{jy} p^W_{kx} (p^I_{Uz} + p^W_{Iz} + 2 p^I_{Uy} q_0 q_1 - 2 p^I_{Uz} q_1^2 - 2 p^I_{Ux} q_0 q_2 - 
   2 p^I_{Uz} q_2^2 + 2 p^I_{Ux} q_1 q_3 + 2 p^I_{Uy} q_2 q_3)\\ & \qquad (p^I_{Ux} + p^W_{Ix} - 2 p^W_{Iz} q_0 q_2 + 
   2 p^W_{Iy} q_1 q_2 - 2 p^W_{Ix} q_2^2 + 2 p^W_{Iy} q_0 q_3 + 2 p^W_{Iz} q_1 q_3 - 2 p^W_{Ix} q_3^2)
\end{align*}

\noindent where $\mathbf{p}^W_j = [0, p^W_{jy}, 0]^T, ~\mathbf{p}^W_j = [p^W_{kx}, p^W_{ky}, 0]^T,  ~\mathbf{p}^I_U = [p^I_{Ux}, p^I_{Uy}, p^I_{Uz}]^T, ~\mathbf{p}^W_I = [p^W_{Ix}, p^W_{Iy}, p^W_{Iz}]^T$ and $\mathbf{q}^W_I = [q_0, q_1, q_2, q_3]^T$. The choice of determinants is governed by the corresponding accelerometer axis excitation. Specifically, $\text{det}(1,2,3,4)$ corresponds to excitation of \{$a_x, a_y$\}, $\text{det}(1,2,3,4)$ corresponds to excitation of \{$a_x, a_z$\} and 
$\text{det}(1,2,3,4)$ corresponds to excitation of \{$a_y, a_z$\}. Under the conditions outlined in Lemma \ref{lem1} (non-collinearity of anchors $i,j,k$), $p^W_{jy} $ and $p^W_{kx}$ cannot be zero. Consider the term:
\begin{align*}
p^W_{Uz} = p^I_{Uz} + p^W_{Iz} + 2 p^I_{Uy} q_0 q_1 - 2 p^I_{Uz} q_1^2 - 2 p^I_{Ux} q_0 q_2 - 
   2 p^I_{Uz} q_2^2 + 2 p^I_{Ux} q_1 q_3 + 2 p^I_{Uy} q_2 q_3
\end{align*}

This term represents the z-coordinate of the position of the mobile radio expressed in the world frame $W$. As per Lemma \ref{lem1}, this cannot be zero. The determinant terms then can be reduced to:
\begin{align*}
\text{det}(1,2,3,4) = p^I_{Uz} + p^W_{Iz} - 2 p^W_{Iy} q_0 q_1 - 
   2 p^W_{Iz} q_1^2 + 2 p^W_{Ix} q_0 q_2 - 2 p^W_{Iz} q_2^2 + 2 p^W_{Ix} q_1 q_3 + 2 p^W_{Iy} q_2 q_3 \\
\text{det}(1,2,3,7) = p^I_{Uy} + p^W_{Iy} + 2 p^W_{Iz} q_0 q_1 - 
   2 p^W_{Iy} q_1^2 + 2 p^W_{Ix} q_1 q_2 - 2 p^W_{Ix} q_0 q_3 + 2 p^W_{Iz} q_2 q_3 - 2 p^W_{Iy} q_3^2 \\
\text{det}(4,5,6,7) = p^I_{Ux} + p^W_{Ix} - 2 p^W_{Iz} q_0 q_2 + 
   2 p^W_{Iy} q_1 q_2 - 2 p^W_{Ix} q_2^2 + 2 p^W_{Iy} q_0 q_3 + 2 p^W_{Iz} q_1 q_3 - 2 p^W_{Ix} q_3^2 
\end{align*}

The terms on the right hand side can be expressed as:
\begin{align*}
\begin{bmatrix}
\text{det}(4,5,6,7)\\
\text{det}(1,2,3,7)\\
\text{det}(1,2,3,4)\\
\end{bmatrix} = \mathbf{p}^I_U + {\mathbf{R}^W_I}^T \mathbf{p}^W_I
\end{align*}

If the three determinants are identically zero then:
\begin{align*}
\mathbf{p}^I_U + {\mathbf{R}^W_I}^T \mathbf{p}^W_I = \mathbf{0}_{3 \times 1}
\end{align*}
Pre-multipliying by $\mathbf{R}^W_I$ and noting that $\mathbf{R}^W_I{\mathbf{R}^W_I}^T = \mathbf{I}_{3 \times 3}$:
\begin{align*}
\mathbf{R}^W_I ( \mathbf{p}^I_U + {\mathbf{R}^W_I}^T \mathbf{p}^W_I ) &= \mathbf{0}_{3 \times 1}\\
\mathbf{R}^W_I \mathbf{p}^I_U + \mathbf{R}^W_I {\mathbf{R}^W_I}^T \mathbf{p}^W_I) &= \mathbf{0}_{3 \times 1}\\
\mathbf{p}^W_I + \mathbf{R}^W_I \mathbf{p}^I_U &= \mathbf{0}_{3 \times 1}
\end{align*}

However, if $\mathbf{p}^W_I + \mathbf{R}^W_I \mathbf{p}^I_U = \mathbf{0}_{1 \times 3}$ then $\delta \mathbf{p}_{ijk}$ is rank deficient and conditions of Lemma \ref{lem1} are violated. Hence, the three determinants cannot be zero simultaneously. This inturn implies that $\mathbf{F}_5  \mathbf{F}_0 - \mathbf{F}_6$ has full column rank when at least two components of $\mathbf{a}_m$ are excited.
\end{proof}

\begin{lemma} For any unit-quaternion $\mathbf{q}$, the 9x3 matrix	$\mathbf{F}_{13}  \mathbf{R}^W_I - \mathbf{F}_{15}$ has full column rank when all three components of $\mathbf{\omega}_m = [\omega_x, \omega_y, \omega_z]^T$ are excited.
\label{lem3}
\end{lemma}

\begin{proof}
 To show that $\mathbf{F}_{13}  \mathbf{R}^W_I - \mathbf{F}_{15}$ is full column rank, it is sufficient to show that determinants of any combination of 3 rows cannot vanish simultaneously. The mathematica script used to compute determinants is provided in appendix B.2. The following determinants are considered:
\begin{align}
\text{det}(1,2,4) &= p^W_{jy} ~ f_1 ~ g_1 \label{eqn:l3t1}\\
\text{det}(1,4,5) &= 2 ~ p^W_{jy} ~ f_1 ~ g_2 \label{eqn:l3t2}\\
\text{det}(1,4,8) &= p^W_{jy} ~ f_1 ~ g_3 \label{eqn:l3t3}\\
\text{det}(1,7,2) &= p^W_{jy} ~ f_2 ~ g_1\\
\text{det}(1,7,5) &= 2 ~ p^W_{jy} ~ f_2 ~ g_2\\
\text{det}(1,8,7) &= p^W_{jy} ~ f_2 ~ g_3\\
\text{det}(2,7,4) &= p^W_{jy} ~ f_3 ~ g_1\\
\text{det}(4,7,5) &= 2 ~ p^W_{jy} ~ f_3 ~ g_2\\
\text{det}(4,7,8) &= p^W_{jy} ~ f_3 ~ g_3
\end{align}

where
\begin{align*}
f_1 &= p^I_{Uz} + p^W_{Iz} - 2 p^W_{Iy} q_0 q_1 - 2 p^W_{Iz} q_1^2 + 2 p^W_{Ix} q_0 q_2 - 2 p^W_{Iz} q_2^2 + 2 p^W_{Ix} q_1 q_3 + 2 p^W_{Iy} q_2 q_3\\
f_2 &= p^I_{Uy} + p^W_{Iy} + 2 p^W_{Iz} q_0 q_1 - 2 p^W_{Iy} q_1^2 + 2 p^W_{Ix} q_1 q_2 - 2 p^W_{Ix} q_0 q_3 + 2 p^W_{Iz} q_2 q_3 - 2 p^W_{Iy} q_3^2\\
f_3 &= p^I_{Ux} + p^W_{Ix} - 2 p^W_{Iz} q_0 q_2 + 2 p^W_{Iy} q_1 q_2 - 2 p^W_{Ix} q_2^2 + 2 p^W_{Iy} q_0 q_3 + 2 p^W_{Iz} q_1 q_3 - 2 p^W_{Ix} q_3^2 \\
g_1 &= p^I_{Uz} + p^W_{Iz} + 2 p^I_{Uy} q_0 q_1 - 2 p^I_{Uz} q_1^2 + 2 p^W_{Ix} q_0 q_2 - 2 p^W_{Iz} q_2^2 - 2 p^W_{Ix} q_1 q_3 - 2 p^I_{Uy} q_2 q_3 - 2 p^I_{Uz} q_3^2 - 2 p^W_{Iz} q_3^2 \\
g_2 &= p^I_{Ux} q_0 q_1 + p^W_{Ix} q_0 q_1 + p^I_{Uz} q_1 q_2 - p^W_{Iz} q_1 q_2 + p^I_{Uz} q_0 q_3 + p^W_{Iz} q_0 q_3 - p^I_{Ux} q_2 q_3 + p^W_{Ix} q_2 q_3 \\
g_3 &= p^I_{Ux} + p^W_{Ix} - 2 p^I_{Ux} q_1^2 - 2 p^W_{Ix} q_1^2 - 
   2 p^W_{Iz} q_0 q_2 - 2 p^I_{Uy} q_1 q_2 - 2 p^W_{Ix} q_2^2 - 2 p^I_{Uy} q_0 q_3 - 2 p^W_{Iz} q_1 q_3 - 2 p^I_{Ux} q_3^2 
\end{align*}

Terms $f_1, f_2, f_3$ can be written as:
\begin{align*}
\begin{bmatrix}
f_1\\
f_2\\
f_3  
\end{bmatrix} &= \mathbf{p}^I_U + {\mathbf{R}^W_I}^T \mathbf{p}^W_I
\end{align*}

If the three terms are identically zero then:
\begin{align*}
\mathbf{p}^I_U + {\mathbf{R}^W_I}^T \mathbf{p}^W_I = \mathbf{0}_{3 \times 1}
\end{align*}
Pre-multipliying by $\mathbf{R}^W_I$ and noting that $\mathbf{R}^W_I{\mathbf{R}^W_I}^T = \mathbf{I}_{3 \times 3}$:
\begin{align*}
\mathbf{R}^W_I ( \mathbf{p}^I_U + {\mathbf{R}^W_I}^T \mathbf{p}^W_I ) &= \mathbf{0}_{3 \times 1}\\
\mathbf{R}^W_I \mathbf{p}^I_U + \mathbf{R}^W_I {\mathbf{R}^W_I}^T \mathbf{p}^W_I) &= \mathbf{0}_{3 \times 1}\\
\mathbf{p}^W_I + \mathbf{R}^W_I \mathbf{p}^I_U &= \mathbf{0}_{3 \times 1}
\end{align*}

However, if $\mathbf{p}^W_I + \mathbf{R}^W_I \mathbf{p}^I_U = \mathbf{0}_{1 \times 3}$ then $\delta \mathbf{p}_{ijk}$ is rank deficient and conditions of Lemma \ref{lem1} are violated. Hence, $f_1, f_2$ and $f_3$ cannot vanish simultaneously. Without loss of generality, it is assumed that $f_1$ does not vanish. With this assumption, the analysis is resitricted to $\text{det}(1,2,4), \text{det}(1,4,5)$ and  $\text{det}(1,4,8)$. Note that $\text{det}(1,2,4)$ corresponds to excitation of $\{ \omega_x, \omega_y\}$, $\text{det}(1,4,5)$ corresponds to excitation of $\{ \omega_x, \omega_y \}$ and  $\text{det}(1,4,8)$ corresponds to excitation of $\{ \omega_x, \omega_y, \omega_z \}$, respectively. Since $f_1$ does not vanish and $p^W_{jy}$ cannot be zero (to satisfy the constraint that the anchors be non-collinear as per Lemma \ref{lem1}), the only way (\ref{eqn:l3t1})-(\ref{eqn:l3t3}) vanish simultaneously is if $g_1, g_2$ and $g_3$ vanish simultaneously. Using the constraint of a unit-quaternion, $q_0^2 + q_1^2 + q_2^2 + q_3^2 = 1$, Terms $g_1, g_2$ and $g_3$ can be rearranged to generate a new constraint:
\begin{align}
p^W_{Iz} + p^I_{Ux} (-2 q_0 q_2 + 2 q_1 q_3) + p^I_{Uy} (2 q_0 q_1 + 2 q_2 q_3) + 
 p^I_{Uz} (q_0^2 - q_1^2 - q_2^2 + q_3^2) = 0
\end{align}

This represents the z-coordinate of the position of the mobile radio in the world frame and cannot be zero as per the constraints of Lemma \ref{lem1}. Hence, the terms $g_1, g_2$ and $g_3$ cannot vanish simultaneously if the constraints of Lemma \ref{lem1} are satisfied. This implies that (\ref{eqn:l3t1})-(\ref{eqn:l3t3}) cannot be zero simultaneously and hence $\mathbf{F}_{13}  \mathbf{R}^W_I - \mathbf{F}_{15}$ has full column rank when all three components of $\mathbf{\omega}_m = [\omega_x, \omega_y, \omega_z]^T$ are excited.
\end{proof}

\begin{lemma} For any unit-quaternion $\mathbf{q}$, the 9x3 matrix	$\mathbf{F}_5  \mathbf{F}_3 + \mathbf{F}_{18}$ has full column rank when all three components of $\mathbf{a}_m = [a_x, a_y, a_z]^T$ are excited.
\label{lem4}
\end{lemma}

\begin{proof}
To show that $\mathbf{F}_5  \mathbf{F}_3 + \mathbf{F}_{18}$ has full column rank, it is sufficient to show that the determinants of any combination of 4 rows cannot vanish simultaneously. The corresponding mathematica script used to compute determinants is provided in appendix B.3. The following determinants are considered:
\begin{align}
\text{det}(1,2,4) &= -p^W_{jy} ~ f_1 ~ g_1 \label{eqn:l4t1}\\
\text{det}(1,4,5) &= -2 ~ p^W_{jy} ~ f_1 ~ g_2 \label{eqn:l4t2}\\
\text{det}(1,4,8) &= -p^W_{jy} ~ f_1 ~ g_3 \label{eqn:l4t3}\\
\text{det}(1,7,2) &= -p^W_{jy} ~ f_2 ~ g_1\\
\text{det}(1,5,7) &= -2 ~ p^W_{jy} ~ f_2 ~ g_2\\
\text{det}(1,8,7) &= -p^W_{jy} ~ f_2 ~ g_3\\
\text{det}(2,7,4) &= -p^W_{jy} ~ f_3 ~ g_1\\
\text{det}(4,7,5) &= -2 ~ p^W_{jy} ~ f_3 ~ g_2\\
\text{det}(4,7,8) &= -p^W_{jy} ~ f_3 ~ g_3
\end{align}
where
\begin{align*}
f_1 &= p^I_{Uz} + p^W_{Iz} - 2 p^W_{Iy} q_0 q_1 - 2 p^W_{Iz} q_1^2 + 2 p^W_{Ix} q_0 q_2 - 2 p^W_{Iz} q_2^2 + 2 p^W_{Ix} q_1 q_3 + 2 p^W_{Iy} q_2 q_3\\
f_2 &= p^I_{Uy} + p^W_{Iy} + 2 p^W_{Iz} q_0 q_1 - 2 p^W_{Iy} q_1^2 + 2 p^W_{Ix} q_1 q_2 - 2 p^W_{Ix} q_0 q_3 + 2 p^W_{Iz} q_2 q_3 - 2 p^W_{Iy} q_3^2\\
f_3 &= p^I_{Ux} + p^W_{Ix} - 2 p^W_{Iz} q_0 q_2 + 2 p^W_{Iy} q_1 q_2 - 2 p^W_{Ix} q_2^2 + 2 p^W_{Iy} q_0 q_3 + 2 p^W_{Iz} q_1 q_3 - 2 p^W_{Ix} q_3^2 \\
g_1 &= p^I_{Uz} + p^W_{Iz} + 2 p^I_{Uy} q_0 q_1 - 2 p^I_{Uz} q_1^2 + 2 p^W_{Ix} q_0 q_2 - 2 p^W_{Iz} q_2^2 - 2 p^W_{Ix} q_1 q_3 - 2 p^I_{Uy} q_2 q_3 - 2 p^I_{Uz} q_3^2 - 2 p^W_{Iz} q_3^2 \\
g_2 &= p^I_{Ux} q_0 q_1 + p^W_{Ix} q_0 q_1 + p^I_{Uz} q_1 q_2 - p^W_{Iz} q_1 q_2 + p^I_{Uz} q_0 q_3 + p^W_{Iz} q_0 q_3 - p^I_{Ux} q_2 q_3 + p^W_{Ix} q_2 q_3 \\
g_3 &= p^I_{Ux} + p^W_{Ix} - 2 p^I_{Ux} q_1^2 - 2 p^W_{Ix} q_1^2 - 
   2 p^W_{Iz} q_0 q_2 - 2 p^I_{Uy} q_1 q_2 - 2 p^W_{Ix} q_2^2 - 2 p^I_{Uy} q_0 q_3 - 2 p^W_{Iz} q_1 q_3 - 2 p^I_{Ux} q_3^2 
\end{align*}

Terms $f_1, f_2, f_3$ can be written as:
\begin{align*}
\begin{bmatrix}
f_1\\
f_2\\
f_3  
\end{bmatrix} &= \mathbf{p}^I_U + {\mathbf{R}^W_I}^T \mathbf{p}^W_I
\end{align*}

If the three terms are identically zero then:
\begin{align*}
\mathbf{p}^I_U + {\mathbf{R}^W_I}^T \mathbf{p}^W_I = \mathbf{0}_{3 \times 1}
\end{align*}
Pre-multipliying by $\mathbf{R}^W_I$ and noting that $\mathbf{R}^W_I{\mathbf{R}^W_I}^T = \mathbf{I}_{3 \times 3}$:
\begin{align*}
\mathbf{R}^W_I ( \mathbf{p}^I_U + {\mathbf{R}^W_I}^T \mathbf{p}^W_I ) &= \mathbf{0}_{3 \times 1}\\
\mathbf{R}^W_I \mathbf{p}^I_U + \mathbf{R}^W_I {\mathbf{R}^W_I}^T \mathbf{p}^W_I) &= \mathbf{0}_{3 \times 1}\\
\mathbf{p}^W_I + \mathbf{R}^W_I \mathbf{p}^I_U &= \mathbf{0}_{3 \times 1}
\end{align*}

However, if $\mathbf{p}^W_I + \mathbf{R}^W_I \mathbf{p}^I_U = \mathbf{0}_{1 \times 3}$ then $\delta \mathbf{p}_{ijk}$ is rank-deficient and conditions of Lemma \ref{lem1} are violated. Hence, $f_1, f_2$ and $f_3$ cannot vanish simultaneously. Without loss of generality, it is assumed that $f_1$ does not vanish. With this assumption, the analysis is resitricted to $\text{det}(1,2,4), \text{det}(1,4,5)$ and  $\text{det}(1,4,8)$. Note that $\text{det}(1,2,4)$ corresponds to excitation of $\{ a_x, a_y\}$, $\text{det}(1,4,5)$ corresponds to excitation of $\{ a_x, a_y \}$ and  $\text{det}(1,4,8)$ corresponds to excitation of $\{ a_x, a_y, a_z \}$ respectively. Since $f_1$ does not vanish and $p^W_{jy}$ cannot be zero (to satisfy the constraint that the anchors be non-collinear as per Lemma \ref{lem1}), the only way (\ref{eqn:l4t1})-(\ref{eqn:l4t3}) vanish simultaneously is if $g_1, g_2$ and $g_3$ vanish simultaneously. Using the constraint of a unit- quaternion, $q_0^2 + q_1^2 + q_2^2 + q_3^2 = 1$, the terms $g_1, g_2$ and $g_3$ can be rearranged to generate a new constraint:
\begin{align}
p^W_{Iz} + p^I_{Ux} (-2 q_0 q_2 + 2 q_1 q_3) + p^I_{Uy} (2 q_0 q_1 + 2 q_2 q_3) + 
 p^I_{Uz} (q_0^2 - q_1^2 - q_2^2 + q_3^2) = 0.
\end{align}

This represents the z-coordinate of the position of the mobile radio in the world frame and cannot be zero as per constraints of Lemma \ref{lem1}. Hence the terms $g_1, g_2$ and $g_3$ cannot vanish simultaneously if the constraints of Lemma \ref{lem1} are satisfied. This implies that (\ref{eqn:l4t1})-(\ref{eqn:l4t3}) cannot be zero simultaneously and hence $\mathbf{F}_{5}  \mathbf{F}_3 + \mathbf{F}_{18}$ has full column rank when all three components of $\mathbf{a}_m = [a_x, a_y, a_z]^T$ are excited.
\end{proof}



\section*{Supplementary material}

# B  Mathematica scripts

## B.1  Script I

```
p = {px, py, pz};
v = {vx, vy, vz};
q = {q0, q1, q2, q3};
ba = {bax, bay, baz};
bw = {bwx, bwy, bwz};
pt = {ptx, pty, ptz};
R = {{q0^2 + q1^2 - q2^2 - q3^2, 2*q1*q2 - 2*q0*q3, 2*q0*q2 + 2*q1*q3},
    {2*q0*q3 + 2*q1*q2, q0^2-q1^2 + q2^2 - q3^2, 2*q2*q3 - 2*q0*q1},
  {2*q1*q3 - 2*q0*q2, 2*q0*q1 + 2*q2*q3, q0^2 - q1^2 - q2^2 + q3^2}};
pwp = {0, 0, 0};
pwq = {0, pwqy, 0};
pwr = {pwrx, pwry, 0};

state = Join[p, v, q, ba, bw, pt];
zv = {0, 0, 0};
zm = {{0, 0, 0}, {0, 0, 0}, {0, 0, 0}};
zm4 = {{0, 0, 0}, {0, 0, 0}, {0, 0, 0}, {0, 0, 0}};
g = {0, 0, 9.8};
Xi = {{-q1, -q2, -q3},
    {q0, -q3, q2},
    {q3, q0, -q1},
    {-q2, q1, q0}};

f0 = Join[v, -R.ba - g, -0.5 * Xi.bw, zv, zv, zv];
f1 = Join[zm, R, zm4, zm, zm, zm];
f2 = Join[zm, zm, 0.5 * Xi, zm, zm, zm];

deltapwp = pwp - R.pt - p;
deltapwq = pwq - R.pt - p;
deltapwr = pwr - R.pt - p;

L0y = 0.5 *{deltapwp.deltapwp, deltapwq.deltapwq, deltapwr.deltapwr};
gradL0y = D[L0y, {state}];

L1f0y = Rationalize[gradL0y.f0];
gradL1f0y = D[L1f0y, {state}];

L1f1L1f0y = Rationalize[gradL1f0y.f1];
gradL1f1L1f0y1 = D[L1f1L1f0y[[All, 1]], {state}];
gradL1f1L1f0y2 = D[L1f1L1f0y[[All, 2]], {state}];
gradL1f1L1f0y3 = D[L1f1L1f0y[[All, 3]], {state}];
gradL1f1L1f0y = Join[gradL1f1L1f0y1, gradL1f1L1f0y2, gradL1f1L1f0y3];
```

```
F0 = D[R.pt, {q}];
F5 = gradL1f1L1f0y[[All, 1 ;; 3]];
F6 = gradL1f1L1f0y[[All, 7 ;; 10]];
F7 = gradL1f1L1f0y[[All, 17 ;; 19]];
mat = F5.F0 - F6;

t1 = Simplify[Det[mat[[{1, 2, 3, 4}, All]]], q0^2 + q1^2 + q2^2 + q3^2 == 1]
t2 = Simplify[Det[mat[[{1, 2, 3, 7}, All]]], q0^2 + q1^2 + q2^2 + q3^2 == 1]
t3 = Simplify[Det[mat[[{4, 5, 6, 7}, All]]], q0^2 + q1^2 + q2^2 + q3^2 == 1]
```

## B.2 Script II

```
p = {px, py, pz} ;
v = {vx, vy, vz} ;
q = {q0, q1, q2, q3} ;
ba = {bax, bay, baz};
bw = {bwx, bwy, bwz} ;
pt = {ptx, pty, ptz} ;
R = {{q0^2 + q1^2 − q2^2 − q3^2, 2 ∗ q1 ∗ q2 − 2 ∗ q0 ∗ q3, 2 ∗ q0 ∗ q2 + 2 ∗ q1 ∗ q3},
    {2 ∗ q0 ∗ q3 + 2 ∗ q1 ∗ q2, q0^2 − q1^2 + q2^2 − q3^2, 2 ∗ q2 ∗ q3 − 2 ∗ q0 ∗ q1},
   {2 ∗ q1 ∗ q3 − 2 ∗ q0 ∗ q2, 2 ∗ q0 ∗ q1 + 2 ∗ q2 ∗ q3, q0^2 − q1^2 − q2^2 + q3^2}};
pwp = {0, 0, 0};
pwq = {0, pwqy, 0};
pwr = {pwrx, pwry, 0};

state = Join[p, v, q, ba, bw, pt];
zv = {0, 0, 0};
zm = {{0, 0, 0}, {0, 0, 0}, {0, 0, 0}};
zm4 = {{0, 0, 0}, {0, 0, 0}, {0, 0, 0}, {0, 0, 0}};
g = {0, 0, 9.8};
Xi = {{−q1, −q2, −q3},
   {q0, −q3, q2},
   {q3, q0, −q1},
   {−q2, q1, q0}};

f0 = Join[v, −R.ba − g, −0.5 ∗ Xi.bw, zv, zv, zv];
f1 = Join[zm, R, zm4, zm, zm, zm];
f2 = Join[zm, zm, 0.5 ∗ Xi, zm, zm, zm];

deltapwp = pwp − R.pt − p;
deltapwq = pwq − R.pt − p;
deltapwr = pwr − R.pt − p;

L0y = 0.5 ∗ {deltapwp.deltapwp, deltapwq.deltapwq, deltapwr.deltapwr};
gradL0y = D[L0y, {state}];

F0 = D[R.pt, {q}];
F3 = 1 / 2 ∗ F0.Xi;

L1f2L0f0y = Rationalize[gradL0y.f2];
gradL1f2L0f0y1 = D[L1f2L0f0y[[All, 1]], {state}];
gradL1f2L0f0y2 = D[L1f2L0f0y[[All, 2]], {state}];
gradL1f2L0f0y3 = D[L1f2L0f0y[[All, 3]], {state}];
gradL1f2L0f0y = Join[gradL1f2L0f0y1, gradL1f2L0f0y2, gradL1f2L0f0y3];
```

```
F13 = gradL1f2L0f0y[[All, 1 ;; 3]];
F15 = gradL1f2L0f0y[[All, 17 ;; 19]];
mat1 = F13.R - F15;
"------"
Factor[Simplify[Simplify[Det[mat1[[{1, 2, 4}, All]]]], q0^2 + q1^2 + q2^2 + q3^2 == 1]]
Factor[Simplify[Simplify[Det[mat1[[{1, 4, 5}, All]]]], q0^2 + q1^2 + q2^2 + q3^2 == 1]]
Factor[Simplify[Simplify[Det[mat1[[{1, 4, 8}, All]]]], q0^2 + q1^2 + q2^2 + q3^2 == 1]]
"-------"
Factor[Simplify[Simplify[Det[mat1[[{1, 7, 2}, All]]]], q0^2 + q1^2 + q2^2 + q3^2 == 1]]
Factor[Simplify[Simplify[Det[mat1[[{1, 7, 5}, All]]]], q0^2 + q1^2 + q2^2 + q3^2 == 1]]
Factor[Simplify[Simplify[Det[mat1[[{1, 8, 7}, All]]]], q0^2 + q1^2 + q2^2 + q3^2 == 1]]
"------"
Factor[Simplify[Simplify[Det[mat1[[{2, 7, 4}, All]]]], q0^2 + q1^2 + q2^2 + q3^2 == 1]]
Factor[Simplify[Simplify[Det[mat1[[{4, 7, 5}, All]]]], q0^2 + q1^2 + q2^2 + q3^2 == 1]]
Factor[Simplify[Simplify[Det[mat1[[{4, 7, 8}, All]]]], q0^2 + q1^2 + q2^2 + q3^2 == 1]]
"-------"
h1 =
 Factor[Simplify[Simplify[Det[mat1[[{1, 2, 4}, All]]]], q0^2 + q1^2 + q2^2 + q3^2 == 1]] /.
  {ptz + pz - 2 py q0 q1 - 2 pz q1^2 + 2 px q0 q2 - 2 pz q2^2 + 2 px q1 q3 + 2 py q2 q3 -> 1, pwqy -> 1}
h2 = Factor[Simplify[Simplify[Det[mat1[[{1, 4, 5}, All]]]],
     q0^2 + q1^2 + q2^2 + q3^2 == 1]] /.
  {ptz + pz - 2 py q0 q1 - 2 pz q1^2 + 2 px q0 q2 - 2 pz q2^2 + 2 px q1 q3 + 2 py q2 q3 -> 1, pwqy -> 1}
h3 = Factor[Simplify[Simplify[Det[mat1[[{1, 4, 8}, All]]]],
     q0^2 + q1^2 + q2^2 + q3^2 == 1]] /.
  {ptz + pz - 2 py q0 q1 - 2 pz q1^2 + 2 px q0 q2 - 2 pz q2^2 + 2 px q1 q3 + 2 py q2 q3 -> 1, pwqy -> 1}
"-------"
h4 = q0^2 + q1^2 + q2^2 + q3^2;
Rest[
  Eliminate[{h1 == 0, h2 == 0, h3 == 0, h4 == 1}, {px, py}] /. {And -> List, Equal -> Subtract}];
h5 = Collect[First@%, {pz, ptx, pty, ptz}, Simplify] /.
  (1 - 2 q1^2 - 2 q2^2) -> Simplify[1 - 2 q1^2 - 2 q2^2 - (1 - h4)]
```

## B.3 Script III

```
p = {px, py, pz} ;
v = {vx, vy, vz} ;
q = {q0, q1, q2, q3} ;
ba = {bax, bay, baz};
bw = {bwx, bwy, bwz} ;
pt = {ptx, pty, ptz} ;
R = {{q0^2 + q1^2 - q2^2 - q3^2, 2 * q1 * q2 - 2 * q0 * q3, 2 * q0 * q2 + 2 * q1 * q3},
    {2 * q0 * q3 + 2 * q1 * q2, q0^2 - q1^2 + q2^2 - q3^2, 2 * q2 * q3 - 2 * q0 * q1},
  {2 * q1 * q3 - 2 * q0 * q2, 2 * q0 * q1 + 2 * q2 * q3, q0^2 - q1^2 - q2^2 + q3^2}};
pwp = {0, 0, 0};
pwq = {0, pwqy, 0};
pwr = {pwrx, pwry, 0};

state = Join[p, v, q, ba, bw, pt];
zv = {0, 0, 0};
zm = {{0, 0, 0}, {0, 0, 0}, {0, 0, 0}};
zm4 = {{0, 0, 0}, {0, 0, 0}, {0, 0, 0}, {0, 0, 0}};
g = {0, 0, 9.8};
Xi = {{-q1, -q2, -q3},
    {q0, -q3, q2},
    {q3, q0, -q1},
    {-q2, q1, q0}};

f0 = Join[v, -R.ba - g, -0.5 * Xi.bw, zv, zv, zv];
f1 = Join[zm, R, zm4, zm, zm, zm];
f2 = Join[zm, zm, 0.5 * Xi, zm, zm, zm];

deltapwp = pwp - R.pt - p;
deltapwq = pwq - R.pt - p;
deltapwr = pwr - R.pt - p;

L0y = 0.5 * {deltapwp.deltapwp, deltapwq.deltapwq, deltapwr.deltapwr};
gradL0y = D[L0y, {state}];

L1f0y = Rationalize[gradL0y.f0];
gradL1f0y = D[L1f0y, {state}];
deltapijk = gradL1f0y[[All, 4 ;; 6]];
deltapijkF = gradL1f0y[[All, 14 ;; 16]];

F0 = D[R.pt, {q}];
F3 = 1 / 2 * F0.Xi;
```

```mathematica
L1f1L1f0y = Rationalize[gradL1f0y.f1];
gradL1f1L1f0y1 = D[L1f1L1f0y[[All, 1]], {state}];
gradL1f1L1f0y2 = D[L1f1L1f0y[[All, 2]], {state}];
gradL1f1L1f0y3 = D[L1f1L1f0y[[All, 3]], {state}];
gradL1f1L1f0y = Join[gradL1f1L1f0y1, gradL1f1L1f0y2, gradL1f1L1f0y3];

L1f0L1f1L0f0y = Rationalize[gradL1f1L1f0y.f0];
gradL1f0L1f1L0f0y = D[L1f0L1f1L0f0y, {state}];
F5 = gradL1f0L1f1L0f0y[[All, 4 ;; 6]];
F18 = gradL1f0L1f1L0f0y[[All, 14 ;; 16]];
Ftest = F5.F3 + F18;

h1 =
 Factor[Simplify[Simplify[Det[Ftest[[{1, 2, 4}, All]]]], q0^2 + q1^2 + q2^2 + q3^2 == 1]] /.
   {ptz + pz - 2 py q0 q1 - 2 pz q1^2 + 2 px q0 q2 - 2 pz q2^2 + 2 px q1 q3 + 2 py q2 q3 → 1, pwqy → 1}
h2 = Factor[Simplify[Simplify[Det[Ftest[[{1, 4, 5}, All]]]],
     q0^2 + q1^2 + q2^2 + q3^2 == 1]] /.
   {ptz + pz - 2 py q0 q1 - 2 pz q1^2 + 2 px q0 q2 - 2 pz q2^2 + 2 px q1 q3 + 2 py q2 q3 → 1,  pwqy → 1}
h3 = Factor[Simplify[Simplify[Det[Ftest[[{1, 4, 8}, All]]]],
     q0^2 + q1^2 + q2^2 + q3^2 == 1]] /.
   {ptz + pz - 2 py q0 q1 - 2 pz q1^2 + 2 px q0 q2 - 2 pz q2^2 + 2 px q1 q3 + 2 py q2 q3 → 1,  pwqy → 1}
"—————"
Factor[Simplify[Simplify[Det[Ftest[[{1, 7, 2}, All]]]], q0^2 + q1^2 + q2^2 + q3^2 == 1]] /.
 {pty + py + 2 pz q0 q1 - 2 py q1^2 + 2 px q1 q2 - 2 px q0 q3 + 2 pz q2 q3 - 2 py q3^2 → 1, pwqy → 1}
Factor[Simplify[Simplify[Det[Ftest[[{1, 7, 5}, All]]]], q0^2 + q1^2 + q2^2 + q3^2 == 1]] /.
 {-pty - py - 2 pz q0 q1 + 2 py q1^2 - 2 px q1 q2 + 2 px q0 q3 - 2 pz q2 q3 + 2 py q3^2 → 1, pwqy → 1}
Factor[Simplify[Simplify[Det[Ftest[[{1, 8, 7}, All]]]], q0^2 + q1^2 + q2^2 + q3^2 == 1]] /.
 {pty + py + 2 pz q0 q1 - 2 py q1^2 + 2 px q1 q2 - 2 px q0 q3 + 2 pz q2 q3 - 2 py q3^2 → 1,  pwqy → 1}
"—————"
Factor[Simplify[Simplify[Det[Ftest[[{2, 7, 4}, All]]]], q0^2 + q1^2 + q2^2 + q3^2 == 1]] /.
 {ptx + px - 2 pz q0 q2 + 2 py q1 q2 - 2 px q2^2 + 2 py q0 q3 + 2 pz q1 q3 - 2 px q3^2 → 1, pwqy → 1}
Factor[Simplify[Simplify[Det[Ftest[[{4, 7, 5}, All]]]], q0^2 + q1^2 + q2^2 + q3^2 == 1]] /.
 {ptx + px - 2 pz q0 q2 + 2 py q1 q2 - 2 px q2^2 + 2 py q0 q3 + 2 pz q1 q3 - 2 px q3^2 → 1, pwqy → 1}
Factor[Simplify[Simplify[Det[Ftest[[{4, 7, 8}, All]]]], q0^2 + q1^2 + q2^2 + q3^2 == 1]] /.
 {ptx + px - 2 pz q0 q2 + 2 py q1 q2 - 2 px q2^2 + 2 py q0 q3 + 2 pz q1 q3 - 2 px q3^2 → 1, pwqy → 1}
"—————"
h4 = q0^2 + q1^2 + q2^2 + q3^2;
Rest[
   Eliminate[{h1 == 0, h2 == 0, h3 == 0, h4 == 1}, {px, py}] /. {And → List, Equal → Subtract}];
h5 = Collect[First@%, {pz, ptx, pty, ptz}, Simplify] /.
  (1 - 2 q1^2 - 2 q2^2) → Simplify[1 - 2 q1^2 - 2 q2^2 - (1 - h4)]
```




\end{document}